\documentclass{article}
\usepackage{iclr2027_conference,times}

\let\cite\citep

\iclrfinalcopy

\usepackage[utf8]{inputenc} %
\usepackage[T1]{fontenc}    %
\usepackage{url}            %
\usepackage{booktabs}       %
\usepackage{amsfonts}       %
\usepackage{amsthm}
\usepackage{nicefrac}       %
\usepackage{microtype}      %
\usepackage{xcolor}         %
\usepackage{algorithm}
\usepackage{algorithmic}
\usepackage{array}
\usepackage{bm}
\usepackage{graphicx}
\usepackage{multirow}
\usepackage{makecell}
\usepackage{caption}
\usepackage{adjustbox}
\usepackage{wrapfig}
\usepackage{float}
\usepackage{pifont}
\usepackage{xcolor}         %
\usepackage{color, colortbl}
\definecolor{citecolor}{HTML}{2980b9}
\definecolor{linkcolor}{HTML}{c0392b}
\definecolor{darkorange}{HTML}{FF8C00}
\definecolor{chocolate}{HTML}{D2691E}
\definecolor{darkgreen}{HTML}{006400}
\definecolor{darkblue}{HTML}{00008B}
\definecolor{mediumblue}{HTML}{0000CD}
\definecolor{dodgerblue}{HTML}{1E90FF}
\definecolor{royalblue}{HTML}{4169E1}
\definecolor{shadecolor}{RGB}{237,237,237}
\definecolor{backred}{RGB}{255, 190, 190}
\definecolor{backblue}{RGB}{210, 230, 250}

\newcommand\methodname{SLIM}

\theoremstyle{plain}
\newtheorem{theorem}{Theorem}[section]
\newtheorem{proposition}[theorem]{Proposition}
\newtheorem{lemma}[theorem]{Lemma}

\theoremstyle{definition}

\newtheorem{assumption}[theorem]{Assumption}
\theoremstyle{remark}

\usepackage{stfloats}
\usepackage{extarrows}
\usepackage[most]{tcolorbox}
\tcbset{
  aibox/.style={
    width=\textwidth,
    top=10pt,
    colback=white,
    colframe=black,
    colbacktitle=black,
    enhanced,
    center,
    attach boxed title to top left={yshift=-0.1in,xshift=0.15in},
    boxed title style={boxrule=0pt,colframe=white,},
  }
}
\newtcolorbox{AIbox}[2][]{aibox,title=#2,#1}

\usepackage{mmstyle}
\usepackage{hyperref}
\title{Dynamic Skill Lifecycle Management for \\Agentic Reinforcement Learning}

\author{%
    Junhao Shen$^{1*}$\quad Teng Zhang$^{2*}$\quad Xiaoyan Zhao$^{1\dag}$\quad  Hong Cheng$^{1}$ \\
    \small{$^1$ The Chinese University of Hong Kong} $\quad$
    \small{$^2$ Lanzhou University} \\
    \small\texttt{\{shen.junhao,TengZhangLZU\}@outlook.com} \quad 
    \small\texttt{\{xzhao,hcheng\}@se.cuhk.edu.hk}
}

\begin{document}

\maketitle
\renewcommand\thefootnote{}
\footnotetext{$\dagger$: Corresponding author. $*$: Equal contribution.}
\footnotetext{}

\begin{abstract}
Large language model agents increasingly rely on external skills to solve complex tasks, where skills act as modular units that extend their capabilities 
beyond what parametric memory alone supports.
Existing skill-based RL methods assume external skills either accumulate as persistent guidance or internalized into the policy, eventually leading to zero-skill inference. We argue this assumption is overly restrictive, since with limited parametric capacity and uneven marginal contribution across skills, the optimal active skill set is non-monotonic, task- and stage-dependent.
In this work, we propose \methodname{}, a framework of dynamic \textbf{S}kill \textbf{LI}fecycle \textbf{M}anagement for agentic reinforcement learning (RL), which dynamically updates the active external skill set as the policy evolves.
Specifically, \methodname{} estimates each active skill’s marginal external contribution through leave-one-skill-out validation, then applies three lifecycle operations: retaining high-value skills, retiring skills whose contribution becomes negligible after sufficient exposure, and expanding the skill bank when persistent failures reveal missing capability coverage.
Experiments show that \methodname{} outperforms the best baselines by an average of 7.1\% points across ALFWorld and SearchQA. 
Results further indicate that policy learning and external skill retention are not mutually exclusive: some skills are absorbed into the policy, while others continue to provide external value, supporting \methodname{} as a more general paradigm for skill-based agentic RL.
Code is available at \url{https://github.com/ejhshen/SLIM}.
\end{abstract}

\begin{figure}[h]
    \centering
    \includegraphics[width=\linewidth]{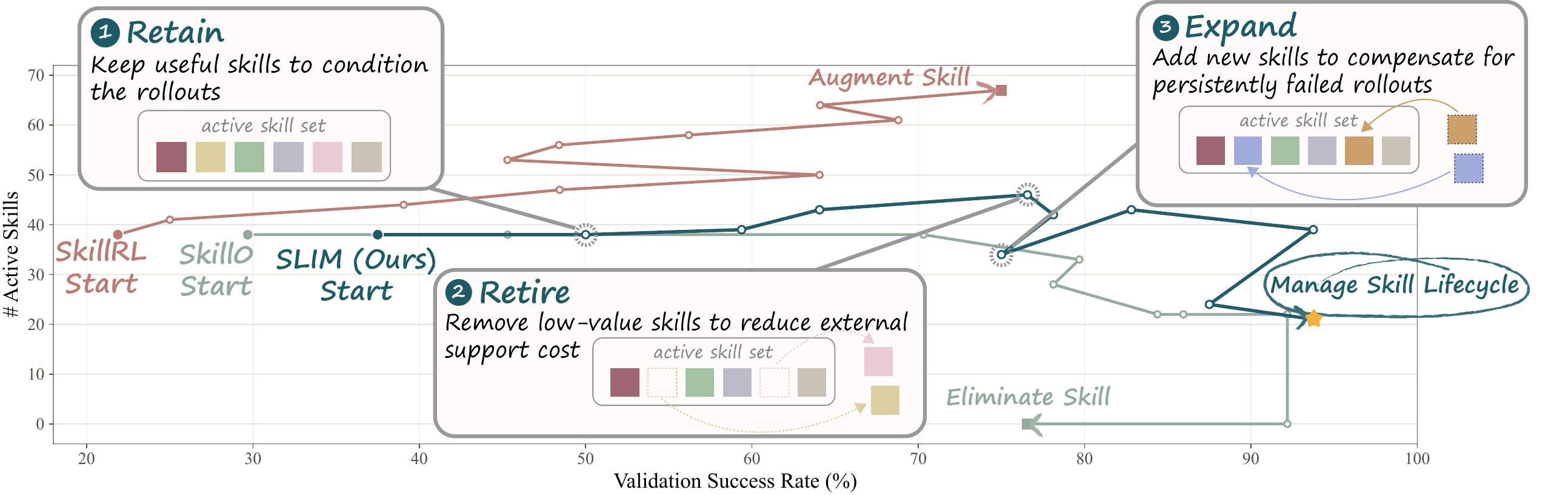}
    \caption{The reinforcement learning dynamics on ALFWorld. We plot validation success rate against the number of skills in active set during training. SkillRL accumulates external skills, whereas Skill0 progressively eliminates them. \methodname{} instead performs retain--retire--expand lifecycle management, converging to a non-empty skill set with higher validation success. This suggests that the effective endpoint is a learned external skill boundary rather than full accumulation or forced elimination.}
    \label{fig:slim_hero}
\end{figure}
\section{Introduction}
Large language model (LLM) agents~\cite{luo2025llmagent_survey,wang2024llm_autonomous_agent_survey} are increasingly used to solve complex tasks that require multi-step reasoning~\cite{plaat2026multistep}, long-horizon planning~\cite{huang2024understanding}, and reliable tool use~\cite{shen2024llm}. A growing way to improve these agents is to equip them with external skills~\cite{chen2026skvm,zheng2025skillweaver,ni2026trace2skill}, where each skill is a modular procedural artifact inserted at inference time to provide reusable task-solving guidance.~\cite{zhou2026externalization}. By conditioning the agent on such external procedural knowledge, skill-based agents can extend capabilities beyond what the base model can reliably express from its parameters alone~\cite{guo2025srsa,wang2026webxskill,yim2026asda}.
Despite this progress, existing skill-based agentic RL methods largely follow two monotonic paradigms. One paradigm treats skills as persistent augmentation and continuously expands the external skill bank to support exploration and decision-making ~\cite{xia2026skillrl}. The other treats skills as temporary scaffolds and gradually removes them toward zero-skill inference, aiming to transfer their benefits into model parameters~\cite{lu2026skill0}. While effective in their respective settings, both approaches  implicitly assume that  the active external skill set should either keep growing or eventually disappear. This assumption overlooks a more general question: \emph{As the agent learns, how should its active 
external skill set evolve under limited parametric capacity and uneven marginal contributions across skills?}

This question is especially important because parametric storage in language models is finite and constrained by model size, training budget, and the trade-off between memorization and generalization~\cite{Zhu20224physics31,Zhu20224physics32,Zhu20224physics33}. 
As a result, not every useful capability should be forced into model parameters.
External skills are particularly suitable for preserving narrow, low-frequency, or long-tail procedures that may be costly or unnecessary to encode parametrically~\cite{zhang2026experience}. 
At the same time, keeping too many skills active is not free since large skill banks can introduce routing noise, and long injected contexts may reduce the reliability of skill use~\cite{zheng2026skillrouter,liu2024lost}.
Therefore, the central problem is not whether skills should be accumulated or eliminated, but how to determine the external boundary of a learning agent.
A skill should be retained when it still provides marginal external value, retired when its contribution becomes negligible, and expanded when persistent failures reveal missing capability coverage.

To address this problem, we propose \methodname{}, a framework for dynamic \textbf{S}kill \textbf{LI}fecycle \textbf{M}anagement in agentic reinforcement learning (RL). \methodname{} dynamically updates the active external skill set as the policy evolves. Specifically, \methodname{} maintains a task-conditioned active skill set during RL, retrieves hierarchical skills from the current active pool, estimates the marginal external contribution (MEC) of each active skill through leave-one-skill-out validation, and couples these signals with RL-based policy optimization to retain, retire, or expand the active skill set over training.
This creates a practical management mechanism between model parameters and external modular skills: reusable capabilities can be absorbed by the policy when external support becomes unnecessary, while narrow or long-tail capabilities can remain external when they continue to provide value. As shown in Figure~\ref{fig:slim_hero}, \methodname{} yields a non-monotonic external capability trajectory rather than forcing full accumulation or zero-skill inference adopted in SkillRL or Skill0.

We evaluate \methodname{} with Qwen3-4B on ALFWorld and SearchQA against standard GRPO~\cite{guo2025deepseek} and representative skill augmentation and internalization methods, including SkillRL~\cite{xia2026skillrl} and Skill0~\cite{lu2026skill0}. \methodname{} outperforms the best baselines by an average of 7.1 percentage points. Additional experiments introduce Qwen3-8B as a new policy backbone and WebShop as a new benchmark (Appendix~\ref{app:cross_backbone_generalization}). Across these settings, \methodname{} learns policy- and task-dependent endpoints, retaining compact skill sets on ALFWorld and WebShop while approaching zero external support on SearchQA.

Our contributions are threefold.
(i) We formulate skill-based agentic RL as a dynamic skill lifecycle management problem, where the active external skill set is not assumed to monotonically grow or vanish, but is treated as an adaptive external capability boundary.
(ii) We propose \methodname{}, which couples GRPO policy learning with MEC-guided retain, retire, and expand updates, adapting external support as the capabilities of the policy evolve.
(iii) Experiments with two policy backbones across three benchmarks show that \methodname{} improves task performance while learning a policy- and task-dependent boundary between internalized capabilities and external skills.
\section{Related Work}
\noindent\textbf{Large Language Model Agents.}
Large language model (LLM) agents turn autoregressive models into sequential decision makers that plan, act, and interact with external environments through tools, APIs, and embodied interfaces~\cite{yao2023react,xie2024openagents,Schick2023toolformer}. Progress in tool use~\cite{patil2024gorilla,qin2024toolllm,hao2023toolkengpt}, web navigation~\cite{he2024webvotager,openai2025deepresearch,google2025geminideepresearch}, computer use~\cite{openai2025computer,anthropic2025computer}, and long-horizon task completion~\cite{shen2026achieving,chen2025internvla,Jin2025SearchR1TL} shows that structured action spaces and external scaffolding are crucial for reliable agent behavior. External memory~\cite{xu2025mem,mem0} and skill support~\cite{chen2026skvm,ni2026trace2skill,wang2026webxskill,zheng2025skillweaver} further improve robustness and compositionality. Our work follows this line but focuses on how the active external skill set should evolve during RL training.

\noindent\textbf{Agentic Reinforcement Learning.}
Reinforcement learning has become a key paradigm for post-training LLM agents~\cite{xu2025towards,lin2025comprehensive}, especially when interaction, exploration, and delayed credit assignment are required~\cite{wei2025reinforcing,goldie2025synthetic,zhao2025r}. Recent methods combine policy optimization with structured rewards, preference signals, or group-relative objectives to improve reasoning and action quality~\cite{guo2025deepseek,shen2026sophia,schulman2017proximal,fujimoto2019off,yu2024dapo}. These advances provide a strong optimization backbone, but they do not determine how external skills should be retained, removed, or expanded during training. \methodname{} keeps the RL optimizer fixed and studies this external capability-management problem as the policy continues to learn.

\noindent\textbf{Skill-Based Agents.}
Skill is a long-standing mechanism for organizing reusable agent behavior into modular procedures for task execution~\cite{zhou2026externalization,zhang2026experience}. Recent LLM-agent work instantiates this idea through external skill banks~\cite{wang2023voyager,wang2024agentworkflowmemory,zheng2025skillweaver,wang2026webxskill,xia2026skillrl}, reusable prompt modules~\cite{fu2024autoguide,liu2024skillact,li2025chatsop,ye2025sopagent}, and distilled procedural guidance~\cite{ni2026trace2skill,chen2026skvm,mi2026procmem}. Closely related methods either keep skills as persistent augmentation~\cite{xia2026skillrl}, eliminate them toward zero-skill inference~\cite{lu2026skill0}, or continuously discover and refine reusable skills from rollout experience~\cite{wu2026cosplay}. \methodname{} is complementary to these directions, \ie, it dynamically updates the active external skill set as the RL policy evolves and decides when skills should be retained, retired, or expanded under finite model capacity.

\section{Preliminaries}
\label{sec:preliminaries}
\noindent\textbf{LLM Agent.}
We model an LLM agent as a policy $\pi_{\theta}$ that interacts with an environment over sequential decisions. Given a task instance $x \sim \mathcal{X}$, the agent produces a trajectory $\tau=(o_1,a_1,\ldots,o_T,a_T)$, where $o_t$ and $a_t$ are the observation and action at step $t$, and $T$ is the horizon. The policy $\pi_{\theta}(a_t \mid h_t)$, parameterized by $\theta$, conditions on the history $h_t=(x,o_1,a_1,\ldots,o_t)$. In text-only environments, both $o_t$ and $a_t$ are token sequences, and $\pi_{\theta}$ is a causal language model that autoregressively generates the next action from $h_t$.

\noindent\textbf{Group Relative Policy Optimization.}
We use Group Relative Policy Optimization (GRPO)~\cite{guo2025deepseek} as the RL optimizer. For each task $x$, GRPO samples $G$ trajectories $\{\tau^{(g)}\}_{g=1}^{G}$ from the behavior policy $\pi_{\theta_{\text{old}}}$ and assigns each a scalar reward $R(\tau^{(g)})$. Let $\mathbf{r}=\{R(\tau^{(1)}),\ldots,R(\tau^{(G)})\}$. The group-relative normalized advantage is
\(
\hat{A}^{(g)}=\frac{R(\tau^{(g)})-\mathrm{mean}(\mathbf{r})}{\mathrm{std}(\mathbf{r})}
\).
Since rewards are outcome-level, the same $\hat{A}^{(g)}$ is used for all action-generation steps in $\tau^{(g)}$. Let $T^{(g)}$ be the number of action steps and $\rho_t^{(g)}(\theta)=\frac{\pi_{\theta}(a_t^{(g)} \mid h_t^{(g)})}{\pi_{\theta_{\text{old}}}(a_t^{(g)} \mid h_t^{(g)})}$ be the step-wise policy ratio. The GRPO objective is
\begin{align}
\resizebox{0.9\linewidth}{!}{$
J_{\text{GRPO}}(\theta)
=
\operatorname*{\mathbb{E}}\limits_{\substack{x \sim \mathcal{X},\\ \{\tau^{(g)}\}_{g=1}^{G} \sim \pi_{\theta_{\text{old}}}(\cdot \mid x)}}
\Bigg[
\frac{1}{G}
\sum\limits_{g=1}^{G}
\frac{1}{T^{(g)}}
\sum\limits_{t=1}^{T^{(g)}}
\Big(
\min
\big(
\rho_t^{(g)}(\theta)\hat{A}^{(g)},
\mathrm{clip}\!\left(\rho_t^{(g)}(\theta),1-\epsilon,1+\epsilon\right)\hat{A}^{(g)}
\big)
- \beta \, \mathbb{D}_{\mathrm{KL}}\!\left[\pi_{\theta} \,\|\, \pi_{\mathrm{ref}}\right](h_t^{(g)})
\Big)
\Bigg],
$}
\label{eq:grpo_obj}
\end{align}
where $\epsilon$ and $\beta$ are hyper-parameters and $\pi_{\mathrm{ref}}$ is the reference policy.

\noindent\textbf{Skill Bank and Problem Setting.}
Following SkillRL~\cite{xia2026skillrl}, we assume a hierarchical external skill library with \emph{general skills} and \emph{task-specific skills}. Let $\mathcal{S}$ denote the global skill bank, with general-skill pool $\mathcal{S}^{\text{gen}}$ and task-specific pool $\mathcal{S}^{k}$ for task type $k$. At audit step $t$, the agent only accesses an active subset $\mathcal{A}_t \subseteq \mathcal{S}$ and acts under a skill-conditioned policy $\pi_{\theta}(a_t \mid h_t,s)$, where $s$ denotes the selected external skill.
Skills outside $\mathcal{A}_t$ are inactive and do not condition the policy. As the policy improves during RL, the usefulness of external skills changes. The problem is to determine which skills should remain active, which can be removed, and when additional skills are needed to support the evolving policy.

\begin{figure}[t]
    \centering
    \includegraphics[width=0.95\linewidth]{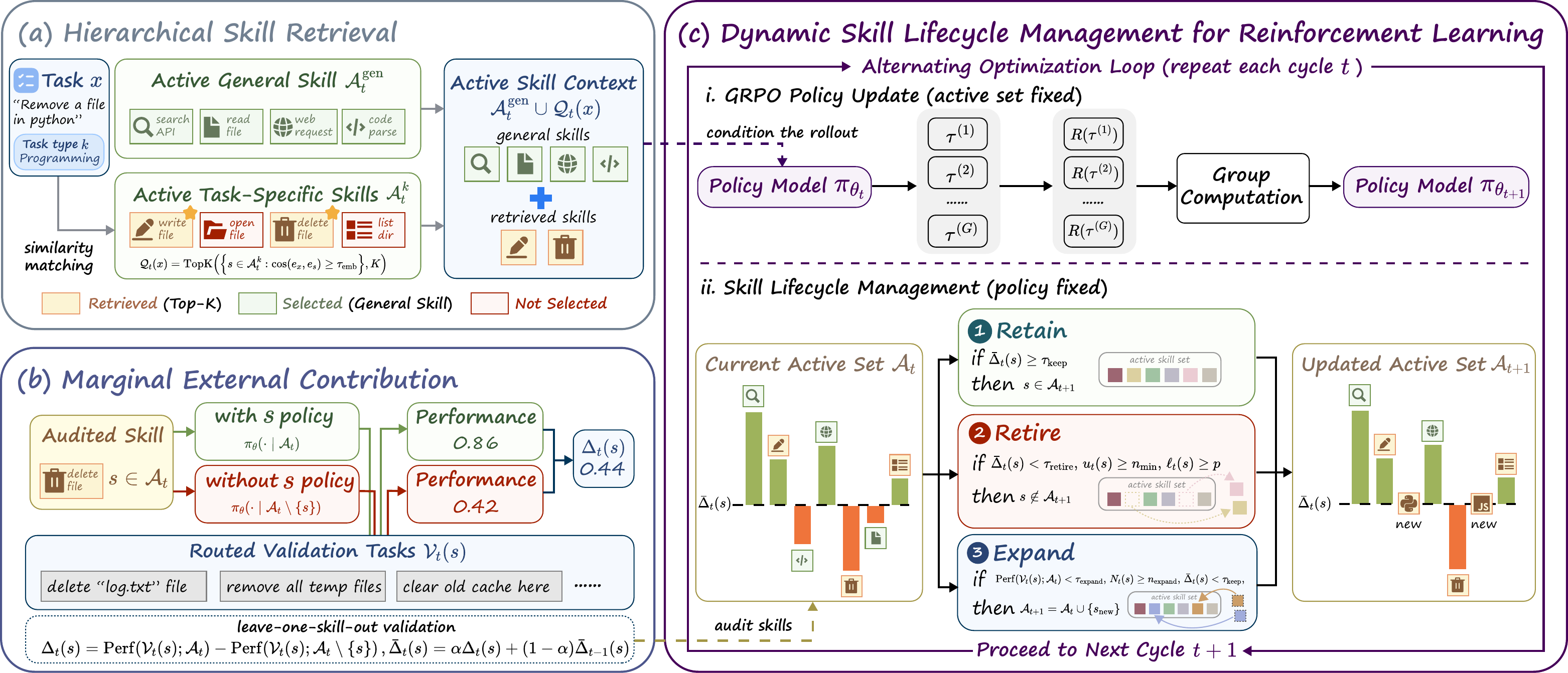}
    \caption{An overview of \methodname{}. \methodname{} first retrieves task-conditioned visible skills, then estimates skill-level marginal contribution via leave-one-skill-out validation, and finally updates the policy and skill lifecycle through GRPO-based retain--retire--expand operations.}
    \vspace{-10pt}
    \label{fig:overview}
\end{figure}

\section{Method: \methodname{}}
An overview of \methodname{} is shown in Figure~\ref{fig:overview}.
\methodname{} addresses the skill management problem in Section~\ref{sec:preliminaries} through three components.
First, it restricts the active-set search to a task-conditioned set of visible skills (Section~\ref{sec:skill_retrieval}).
Next, it estimates the local value of each audited skill through leave-one-skill-out validation (Section~\ref{sec:marginal_contribution}).
Finally, it combines these signals with GRPO-based policy optimization, enabling the active skill set to be retained, retired, or expanded as training proceeds (Section~\ref{sec:lifecycle_rl}). In this way, \methodname{} learns which capabilities should remain active and which should be removed from active external support.

\subsection{Hierarchical Skill Retrieval}
\label{sec:skill_retrieval}
The first component of \methodname{} restricts skill selection to task-conditioned candidates. Directly selecting from the full skill bank is a combinatorial problem, so \methodname{} uses the hierarchical setup in Section~\ref{sec:preliminaries} to convert global skill selection into task-conditioned candidate selection.

Formally, let $\mathcal{A}^{\text{gen}}_t \subseteq \mathcal{S}^{\text{gen}}$ denote the currently active general-skill pool at audit step $t$, and let $\mathcal{A}^{k}_t \subseteq \mathcal{S}^{k}$ denote the active task-specific pool for task type $k$. For a task instance $x$ of type $k$, \methodname{} selects the active general skills together with a retrieved task-specific subset from $\mathcal{A}^{k}_t$. Let $e_x$ denote the embedding of the current task description and let $e_s$ denote the embedding of skill $s$. The retrieved task-specific skill set is
\begin{align}
\mathcal{Q}_t(x)
=
\operatorname{TopK}
\left(
\left\{
s \in \mathcal{A}^{k}_t :
\cos(e_x,e_s)\ge \tau_{\text{emb}}
\right\},
K
\right),
\label{eq:skill_retrieval}
\end{align}
where $\tau_{\text{emb}}$ is the retrieval threshold and $K$ is the maximum number of task-specific skills loaded into the prompt. The final skill-conditioned policy for task $x$ is thus the union of the active general skills and the retrieved task-specific set, \ie, $\pi_\theta(a_t\mid h_t, \mathcal{A}^{\text{gen}}_t\cup\mathcal{Q}_t(x))$. Because retrieval is restricted to the current active set, lifecycle decisions directly affect the external capability exposed to later rollouts.

Intuitively, external skills must be relevant before they can be useful. At the same time, retrieval relevance alone does not tell us whether keeping a skill external is still worthwhile. Different active skills may be selected for the same type of tasks while contributing very different amounts of external value. This motivates an explicit estimate of the marginal external contribution of each active skill.

\subsection{Marginal External Contribution Estimation}
\label{sec:marginal_contribution}
Given the routed skill set from Section~\ref{sec:skill_retrieval}, the next problem is to decide whether each active skill still deserves external support. Even after restricting the candidate set, enumerating skill combinations to estimate the marginal external contribution (MEC) of each active skill remains impractical. \methodname{} therefore uses leave-one-skill-out validation as a tractable local approximation.

For an audited skill $s \in \mathcal{A}_t$, let $\mathcal{V}_t(s)$ denote the subset of tasks from the development split whose rollouts use skill $s$ under the current active set, \ie, tasks $x$ for which $s \in \mathcal{A}^{\text{gen}}_t\cup\mathcal{Q}_t(x)$. Let $\operatorname{Perf}(\mathcal{V};\mathcal{A})$ denote the validation performance on subset $\mathcal{V}$ when the active set is $\mathcal{A}$. The MEC of $s$ at audit step $t$ is defined by leave-one-skill-out validation:
\begin{align}
\Delta_t(s)=
\operatorname{Perf}\!\left(\mathcal{V}_t(s);\mathcal{A}_t\right)
-
\operatorname{Perf}\!\left(\mathcal{V}_t(s);\mathcal{A}_t \setminus \{s\}\right).
\label{eq:marginal_contribution}
\end{align}
To reduce audit noise, \methodname{} smooths current-round estimates with an exponential moving average, $\bar{\Delta}_t(s)=\alpha \Delta_t(s)+(1-\alpha)\bar{\Delta}_{t-1}(s)$. We use $\bar{\Delta}_t(s)$ rather than $\Delta_t(s)$ for lifecycle management. A positive value means the current policy still benefits from keeping that capability external, while a near-zero or negative value means the capability may have been absorbed, become redundant, or become harmful as an external aid. This is a local estimate conditioned on the current policy, active set, and routing behavior, not a global attribution over all possible skill subsets. It is reliable when tasks from the development split routed to $s$ reflect the same local behavior seen during rollout; in that case, removing $s$ on those tasks is a direct test of whether the policy still needs its external support. Lemma~\ref{lem:retention_proxy} in Appendix~\ref{app:theory} further explains this local surrogate.

\subsection{Dynamic Skill Lifecycle Management for Reinforcement Learning}
\label{sec:lifecycle_rl}
We now couple skill lifecycle updates with policy optimization through alternating updates. \methodname{} decomposes each audit cycle into GRPO policy updates with the active set fixed, followed by skill lifecycle management with the policy fixed.

In the GRPO stage, $\mathcal{A}_t$ is fixed, and the policy is updated using $J_{\text{GRPO}}(\theta;\mathcal{A}_t)$ under the current external support. This step may reduce the dependence of the policy on some external skills, but whether such dependence has actually disappeared is measured by MEC rather than assumed.

In the skill lifecycle management stage, $\theta_{t+1}$ is fixed. \methodname{} uses the smoothed MEC $\bar{\Delta}_t(s)$ and routed failures to retain, retire, or expand skills. Candidates are audited sequentially; after each retirement, routing and the validation baseline are refreshed before the next candidate is evaluated.

\emph{Retain} keeps an audited skill $s$ active when its smoothed MEC is clearly positive. Here $\tau_{\mathrm{keep}}$ specifies the contribution level for retaining $s$ as external support in later rollouts.
\begin{align}
\text{if}\, \bar{\Delta}_t(s) \ge \tau_{\mathrm{keep}}, \, \text{then}\, s \in \mathcal{A}_{t+1}.
\label{eq:retain_rule}
\end{align}
\emph{Retire} removes an audited skill $s$ when its marginal contribution becomes negligible and this signal remains stable after sufficient exposure. Here  $u_t(s)$ is the cumulative exposure count and $\ell_t(s)$ is the low-contribution streak. These two conditions protect low-frequency skills from being removed before enough routed evidence is observed. The threshold $\tau_{\mathrm{retire}}$ specifies the level below which the measured contribution is treated as negligible. Retiring $s$ only means that it no longer provides enough marginal value under the current policy; it may have been internalized, become redundant, or become noisy or obsolete.
When $\tau_{\mathrm{retire}}\le\bar{\Delta}_t(s)<\tau_{\mathrm{keep}}$, \methodname{} makes no immediate lifecycle transition for $s$ and keeps it active until later audits provide stronger evidence.
\begin{align}
\text{if}\, \bar{\Delta}_t(s) < \tau_{\mathrm{retire}},\,
u_t(s)\ge n_{\min},\,
\ell_t(s)\ge p,\,
\text{then}\, s \notin \mathcal{A}_{t+1}.
\label{eq:retire_rule}
\end{align}
\emph{Expand} adds a new skill $s_{\text{new}}$ when the current active skill $s$ persistently fails to cover its routed task region. Here $N_t(s)$ is the accumulated number of task failures routed to $s$. The threshold $\tau_{\mathrm{expand}}$ identifies low-performing task regions whose failure traces guide the creation of new skills.
\begin{align}
\text{if}\,
\operatorname{Perf}(\mathcal{V}_t(s); \mathcal{A}_t) < \tau_{\mathrm{expand}},\,
N_t(s)\ge n_{\text{expand}},\,
\bar{\Delta}_t(s)<\tau_{\mathrm{keep}},\,
\text{then}\,
\mathcal{A}_{t+1} = \mathcal{A}_t \cup \{s_{\text{new}}\},
\label{eq:expand_rule}
\end{align}
Theorem~\ref{thm:sequential_retirement} bounds the performance loss of sequential retirements under the audit accuracy and routing conditions in Appendix~\ref{app:theory}, while Lemma~\ref{lem:useful_skill_protection} characterizes when a currently useful skill is protected from retirement. Intuitively, if the policy still depends on a necessary skill, removing it hurts validation performance, $\bar{\Delta}_t(s)$ remains high, and retirement is blocked; if $\bar{\Delta}_t(s)$ stays near zero, active retention is unnecessary because the skill may have been internalized or become redundant.
Additionally, \methodname{} subsumes prior methods as boundary cases. If retirement is disabled, i.e., $\mathcal{A}_{t+1}\supseteq\mathcal{A}_t$ for all $t$, it reduces to a SkillRL-like persistent augmentation regime. The active set then grows without removing skills whose contribution has diminished. If expansion is disabled and retirement is enforced until $\mathcal{A}_t=\varnothing$, it reduces to a Skill0-like zero-skill regime. This regime removes all external support, including skills that may still benefit the current policy.

\section{Implementation}
\begin{wrapfigure}{r}{0.45\textwidth}
\vspace{-1em}
\begin{minipage}{0.45\textwidth}
\noindent\rule{\linewidth}{0.8pt}
\vspace{-15pt}
\captionsetup{labelfont=bf}
\captionof{algorithm}{Practical SLIM Training Loop}
\label{alg:slim_implementation}
\vspace{-4pt}
\noindent\rule{\linewidth}{0.4pt}
\vspace{-12pt}
\footnotesize
\begin{algorithmic}[1]
\REQUIRE Initial policy $\pi_\theta$, skill bank $\mathcal{S}$, active set $\mathcal{A}_0$, train split, development split, retrieval cap $K$, audit interval $d$, task-specific audit budget $M$, expansion budget $B$
\ENSURE Trained policy $\pi_\theta$, final active set $\mathcal{A}_T$, retired skills, expanded skills, lifecycle logs
\FOR{GRPO step $r=1,\ldots,T$}
    \STATE Sample tasks, retrieve $\mathcal{Q}_t(x)$ by Eq.~\eqref{eq:skill_retrieval}, and roll out $\pi_\theta(\cdot \mid h,\mathcal{A}^{\mathrm{gen}}_t\cup\mathcal{Q}_t(x))$.
    \STATE Update $\theta$ with GRPO on the rollouts.
    \IF{$r \bmod d = 0$}
        \STATE Run validation with current routing; record routed skills, outcomes, and routed failures.
        \STATE Select audit candidates within budget, including top-$M$ skills by recent routed usage.
        \FOR{each audited skill $s$}
            \STATE Compute leave-one-skill-out $\Delta_t(s)$ by Eq.~\eqref{eq:marginal_contribution}.
            \STATE Update $\bar{\Delta}_t(s)$ and lifecycle statistics.
            \STATE Apply Eqs.~\eqref{eq:retain_rule}--\eqref{eq:retire_rule}; after retirement, refresh routing and the full-active-set validation baseline.
        \ENDFOR
        \STATE For each qualifying failure bucket, create up to $B$ task-specific skills by Eq.~\eqref{eq:expand_rule}; update $\mathcal{A}_t$.
    \ENDIF
\ENDFOR
\end{algorithmic}
\vspace{-8pt}
\noindent\rule{\linewidth}{0.8pt}
\end{minipage}
\vspace{-1.2em}
\end{wrapfigure}

\noindent\textbf{Algorithm.}
Algorithm~\ref{alg:slim_implementation} summarizes the practical training loop of \methodname{}. The implementation follows the three components in Section~\ref{sec:skill_retrieval}--\ref{sec:lifecycle_rl}, \ie, each GRPO step retrieves active skills, performs skill-conditioned rollouts, and updates the policy; every audit interval, \methodname{} estimates marginal external contribution and applies retain, retire, or expand operations. To keep auditing affordable, \methodname{} does not evaluate every active skill. Lifecycle audits are performed every $d=10$ GRPO steps, and each audit considers at  most $M=4$ skills with the highest recent routed usage among skills that appeared in top-$K$ retrieval.

\noindent\textbf{Training and Inference Settings.}
For training, task-specific retrieval uses Qwen3-Embedding-0.6B~\cite{zhang2025qwen3embedding} with $K=3$ and $\tau_{\mathrm{emb}}=0.45$. We optimize the policy with GRPO using outcome-level rewards. Specifically, each completed rollout receives the environment success reward, with invalid-action penalties applied during trajectory collection. In the main \methodname{} runs, we disable both policy-side KL loss and KL-in-reward regularization. Expansion uses o3 with an Anthropic-style skill-creator workflow~\cite{anthropic2025skills} to create standalone task-specific \texttt{SKILL.md} artifacts from routed failure buckets, matching SkillRL's default creator choice. During final inference, the agent can run with skills by retrieving active skills before each rollout; the prompt contains the active general skills and the retrieved task-specific set $\mathcal{Q}_T(x)$, and no lifecycle update is performed. Prompt templates, lifecycle thresholds, full training settings, and inference details are provided in Appendix~\ref{sec:slim_setup}.

\section{Experiment}
\label{sec:experiment}
\subsection{Evaluation Setup}
\noindent\textbf{Benchmarks Baselines.}
We conduct all main experiments with Qwen3-4B~\cite{qwen3technical} on ALFWorld~\cite{shridhar2020alfworld} and SearchQA~\cite{Jin2025SearchR1TL}. ALFWorld covers Pick, Look, Clean, Heat, Cool, and Pick2 household tasks, while SearchQA covers NQ~\cite{kwiatkowski2019nq}, TriviaQA~\cite{joshi2017triviaqa}, PopQA~\cite{mallen2023popqa}
, HotpotQA~\cite{yang2018hotpotqa}, 2Wiki~\cite{ho20202wiki}, MuSiQue~\cite{trivedi2022musique}, and Bamboogle~\cite{press2023bamboogle}. We compare against prompt-based, agent/memory-based, and RL-based baselines, including ReAct~\cite{yao2023react}, Reflexion~\cite{shinn2023reflexion}, Mem0~\cite{mem0}, ExpeL~\cite{zhao2024expel}, GRPO~\cite{shao2024deepseekmath}, EvolveR~\cite{wu2025evolver}, SkillRL~\cite{xia2026skillrl}, and Skill0~\cite{lu2026skill0}. Full baseline details are provided in Appendix~\ref{sec:baseline_setup}.

\noindent\textbf{Evaluation.}
We report success rates under the same benchmark success criteria, i.e., objective completion on ALFWorld and correct final answers on SearchQA. Additionally, all methods share the same protocol. Policy updates use the train split, hyperparameters are selected on a smaller subset of the train split, skill management uses the development split, and final reporting uses the test split. Online-RL methods use the final checkpoint after a fixed training horizon, without cold-start SFT or warmup. Appendix~\ref{app:additional_results} reports additional backbone and benchmark evaluations, independent training runs, skill-bank transfer, sensitivity, hyperparameters, expanded baselines, and audit overhead. Detailed splits and fairness controls are in Appendix~\ref{app:setup}.

\subsection{Main Results}
\begin{table*}[t]
\centering
\definecolor{slimBestRed}{RGB}{232,204,198}
\definecolor{slimSecondBlue}{RGB}{202,216,224}
\newcommand{\bestcell}[1]{\cellcolor{slimBestRed}\textbf{#1}}
\newcommand{\secondcell}[1]{\cellcolor{slimSecondBlue}\underline{#1}}
\caption{Main results on ALFWorld and SearchQA. All entries report success rate. $^\dagger$ denotes evaluation with retrieved external skills and applies throughout unless otherwise specified. Avg. denotes the overall success rate, also used below. \textbf{\colorbox{slimBestRed}{Best}} and \underline{\colorbox{slimSecondBlue}{second-best}} are highlighted.}
\label{tab:main_results}
\small
\begin{adjustbox}{max width=0.90\textwidth}
\begin{tabular}{l|ccccccc|cccccccc}
\toprule
\multirow{2}{*}{Method}
& \multicolumn{7}{c|}{ALFWorld}
& \multicolumn{8}{c}{SearchQA} \\
\cmidrule(lr){2-8}\cmidrule(lr){9-16}
& Pick & Look & Clean & Heat & Cool & Pick2 & Avg.
& NQ & TriviaQA & PopQA & HotpotQA & 2Wiki & MuSiQue & Bamboogle & Avg. \\
\midrule
\multicolumn{16}{c}{\textit{Prompt-based methods}} \\
\midrule
Zero-Shot & 82.9 & 0.0 & 31.2 & 23.1 & 21.7 & 30.0 & 41.4 & 28.5 & 49.6 & 32.7 & 23.5 & 27.7 & 5.3 & 35.2 & 32.3 \\
Few-Shot & 80.0 & 40.0 & 37.5 & 7.7 & 17.4 & 15.0 & 39.1 & 32.9 & 55.6 & 34.9 & 28.0 & 28.5 & 7.2 & \secondcell{37.6} & 35.5 \\
Zero-Shot$^\dagger$ & 94.3 & 40.0 & 81.2 & 38.5 & 17.4 & 55.0 & 63.3 & 26.6 & 48.2 & 31.8 & 22.7 & 27.9 & 5.3 & 33.6 & 31.5 \\
Few-Shot$^\dagger$ & 94.3 & 60.0 & \secondcell{84.4} & 38.5 & 34.8 & 30.0 & 64.1 & 32.3 & 54.6 & 34.6 & 28.1 & 27.2 & 7.4 & \bestcell{38.4} & 34.8 \\
\midrule
\multicolumn{16}{c}{\textit{Agent- or memory-based methods}} \\
\midrule
ReAct & \bestcell{100.0} & 33.3 & 25.0 & 53.8 & 23.8 & 26.7 & 50.5 & 33.2 & 54.5 & 36.5 & 28.8 & 30.8 & 7.7 & 33.6 & 36.4 \\
Reflexion & 93.9 & 30.0 & 52.0 & 20.0 & 47.8 & 68.2 & 59.4 & 22.3 & 45.7 & 25.8 & 22.6 & 28.3 & 4.9 & 31.2 & 29.1 \\
Mem0 & 93.9 & 30.0 & 32.0 & 26.7 & 17.4 & 18.2 & 42.2 & 29.3 & 50.6 & 33.3 & 24.6 & 28.2 & 6.5 & 30.4 & 33.1 \\
ExpeL & \secondcell{97.1} & 62.5 & 38.5 & 13.3 & 21.7 & 61.9 & 53.5 & 23.5 & 47.0 & 29.6 & 22.8 & 30.0 & 6.5 & 28.0 & 31.0 \\
\midrule
\multicolumn{16}{c}{\textit{RL-based methods}} \\
\midrule
GRPO & 85.4 & \bestcell{100.0} & 49.8 & 64.6 & 53.1 & 54.2 & 67.2 & 35.9 & 57.8 & 36.5 & 30.8 & 30.1 & 9.2 & 30.4 & 37.5 \\
GRPO$^\dagger$ & 84.6 & 62.5 & 48.9 & 68.8 & 61.0 & 80.8 & 68.8 & 36.4 & 58.2 & 37.0 & 31.2 & 30.4 & 9.6 & 31.2 & 37.9 \\
EvolveR & 67.6 & 37.5 & 49.3 & 15.6 & 36.0 & 36.5 & 39.8 & 36.0 & 57.5 & 36.8 & 30.6 & 30.0 & 9.5 & 29.6 & 37.4 \\
SkillRL$^\dagger$ & 90.6 & 75.0 & 54.6 & \secondcell{76.2} & 67.7 & \bestcell{87.5} & \secondcell{75.0} & 36.8 & 59.8 & 36.9 & 31.5 & 29.7 & 10.3 & 31.2 & 38.1 \\
Skill0 & 93.6 & \secondcell{87.5} & 59.8 & 66.7 & 57.6 & 78.4 & 74.2 & 37.9 & 59.5 & 38.6 & 32.7 & \bestcell{31.9} & 10.3 & 32.8 & \secondcell{39.3} \\
\methodname{} & 91.4 & 80.0 & 46.9 & 61.5 & \secondcell{73.9} & \secondcell{85.0} & 72.7 & \bestcell{38.6} & \bestcell{62.2} & \secondcell{40.0} & \bestcell{37.2} & \secondcell{31.7} & \bestcell{12.8} & \secondcell{37.6} & \bestcell{41.0} \\
\methodname{}$^\dagger$ & 92.9 & \bestcell{100.0} & \bestcell{91.4} & \bestcell{78.3} & \bestcell{88.5} & 81.2 & \bestcell{87.5} & \secondcell{38.4} & \secondcell{62.1} & \bestcell{40.4} & \secondcell{36.9} & 31.5 & \secondcell{12.7} & 36.0 & \bestcell{41.0} \\
\bottomrule
\end{tabular}
\end{adjustbox}
\vspace{-0.5 em}
\end{table*}

\begin{figure}[t]
    \centering
    \includegraphics[width=0.8\linewidth]{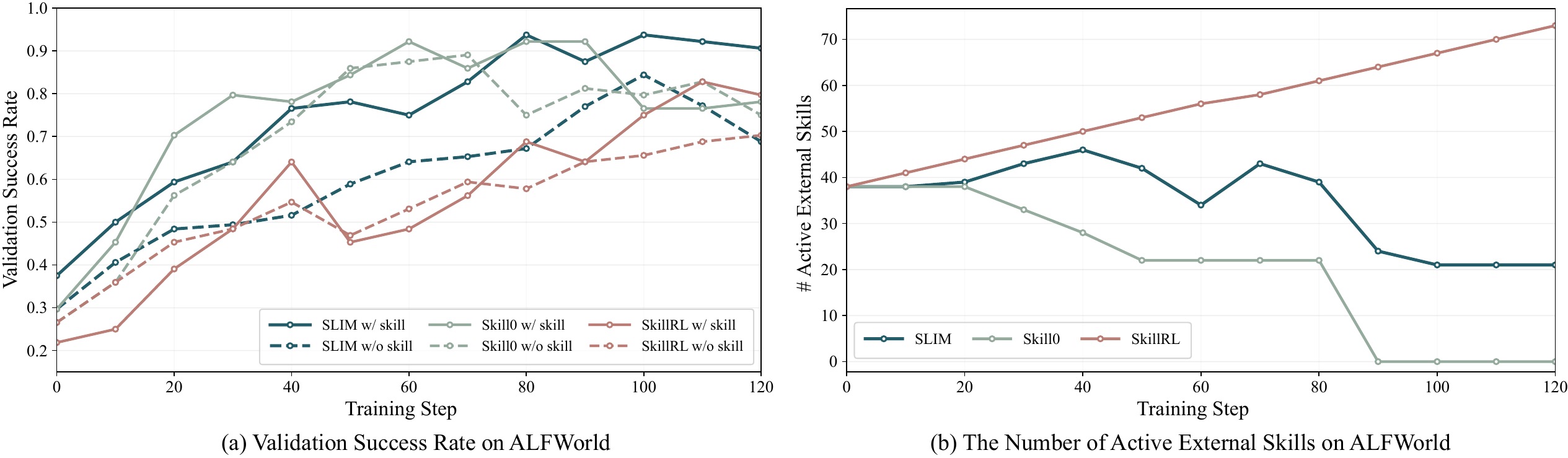}
    \caption{Training dynamics on ALFWorld. Panel (a) compares with-skill and no-skill evaluation curves for \methodname{}, SkillRL and Skill0, where solid lines denote evaluation with external skills, while dashed lines denote without external skills. Panel (b) tracks the number of active skills over training.}
    \vspace{-16pt}
    \label{fig:training_dynamics}
\end{figure}

\noindent\textbf{Overall Comparison.}
Table~\ref{tab:main_results} reports the main comparison. On ALFWorld, \methodname{}$^\dagger$ reaches 87.5, outperforming the strongest non-SLIM baseline, SkillRL$^\dagger$, by 12.5 points. It also substantially improves over GRPO, GRPO$^\dagger$, and Skill0, showing that the gain is not produced by ordinary RL, naive skill injection, persistent accumulation, or forced zero-skill inference alone. The same trained \methodname{} policy achieves 72.7 without skills and 87.5 with its final active skill set, indicating that ALFWorld contains long-horizon procedural behaviors where some capabilities remain better kept externally.
SearchQA shows a different regime. \methodname{} and \methodname{}$^\dagger$ both reach 41.0, exceeding Skill0 by 1.7 points. Their matching performance indicates that the gain is largely reflected in the trained policy, with little final external dependence. Together, these results show task-dependent endpoints: retained procedural support on ALFWorld and near-zero dependence on SearchQA. Across three independent training runs with each backbone, \methodname{} achieves the highest mean success rate on both benchmarks, with stable lifecycle trajectories on ALFWorld (Appendix~\ref{app:robustness}).

\noindent\textbf{Detailed Analysis.}
The ALFWorld gains concentrate on procedural state-transformation tasks. \methodname{}$^\dagger$ reaches 91.4 on Clean and 88.5 on Cool, far above SkillRL$^\dagger$ and Skill0. This aligns with the lifecycle probe in Figure~\ref{fig:lifecycle_analysis}, \ie, the clean-specific skill \texttt{cle\_003} remains externally valuable, while several cooling skills are retired, suggesting that the gain comes from filtering procedural support rather than preserving every task-specific skill. Heat shows a smaller but consistent improvement, indicating that lifecycle management is most useful when procedures are compositional and unevenly covered.
The advantage is not uniform across all task types. Simpler object-acquisition tasks such as Pick leave limited headroom for lifecycle management, while \methodname{}$^\dagger$ reaches full success on Look by retaining skills that cover this task type. More importantly, GRPO$^\dagger$ improves Pick2 and Cool but hurts or does not improve Look and Pick, showing that naive skill insertion is not uniformly beneficial. This supports our lifecycle control, where skills remain active only when their marginal effect is positive under the current policy.
On SearchQA, the improvement is smaller but broadly distributed: \methodname{} or \methodname{}$^\dagger$ is best or near-best on most subsets. Naive skill insertion can even hurt, as Zero-Shot$^\dagger$ and Few-Shot$^\dagger$ fall below their no-skill counterparts. Thus, SearchQA reflects a lower-dependence regime where lifecycle-guided training improves the policy more than inference-time skill insertion, complementing ALFWorld where retained procedural skills remain valuable.

\subsection{Training Dynamics}

\begin{wrapfigure}{r}{0.48\textwidth}
\centering
\vspace{-1em}
\captionof{table}{Ablation study on ALFWorld. All variants are trained under the same settings and evaluation protocol as SLIM (with skill).}
\label{tab:ablation_lifecycle}
\resizebox{0.65\linewidth}{!}{
\small
\setlength{\tabcolsep}{4pt}
\begin{tabular}{l|c}
\toprule
Method & ALFWorld Avg. \\
\midrule
\methodname{} & \textbf{87.5} \\
\midrule[0.5pt]
w/o Retirement & 73.4 \\
w/o Expansion & 78.9 \\
Random Audit & 68.8 \\
Fixed Active Set Size & 75.6 \\
\bottomrule
\end{tabular}

}
\includegraphics[width=0.75\linewidth]{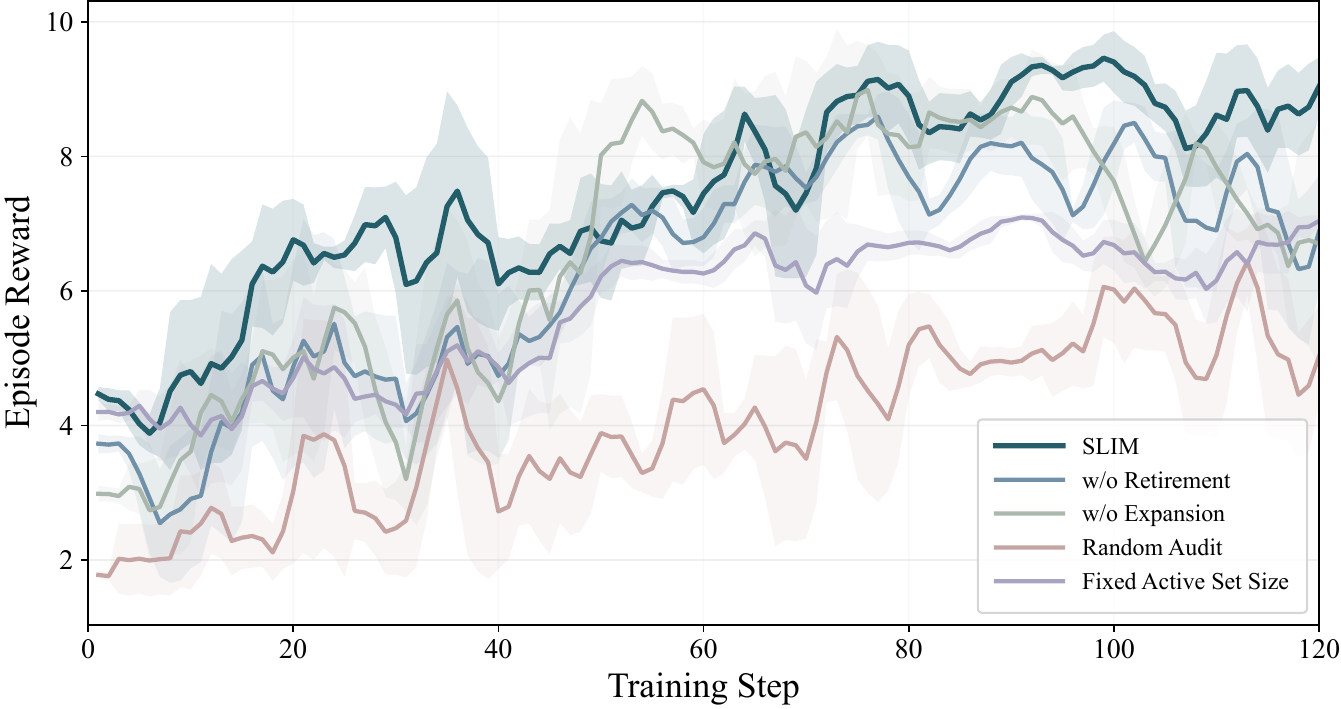}
\captionof{figure}{Training reward dynamics of SLIM and its ablation variants on ALFWorld. For readability, we apply a centered moving average with a window size of 5 training steps. The shaded region is the local variation within the window.}
\label{fig:ablation_reward}
\vspace{-1.2em}
\end{wrapfigure}

Figure~\ref{fig:training_dynamics} compares training dynamics on ALFWorld. SkillRL follows persistent accumulation where its active skill count grows from 38 to 73 throughout training. Although both its with-skill and no-skill validation curves improve, its final performance remains below \methodname{}. This shows that keeping more external skills does not necessarily yield better final performance, consistent with our motivation that large skill banks can introduce routing and context overhead.
Skill0 exhibits the opposite pattern. Its active set decreases from 38 to 0. Its no-skill performance becomes strong in the later stage, indicating that the policy indeed learns from skill-conditioned training. However, once the active skill set reaches zero around epoch 90, validation drops from 92.2\% to 76.6\% in the following audit interval. This supports our claim that retirement is not equivalent to successful internalization since forced zero-skill inference may remove useful external support, especially for unstable, low-frequency or long-tail capabilities.

\methodname{} learns a non-monotonic trajectory. Its active set first expands from 38 to 46, fluctuates as expansion and retirement alternate, and finally stabilizes at a compact non-empty set of 21 skills. Meanwhile, no-skill performance rises from 29.7\% to 84.4\%, showing that the policy itself is learning, while with-skill performance peaks at 93.8\% and remains 90.6\% at the end. Thus, \methodname{} does not trade policy learning for external dependence. It improves the policy while preserving a set of external skills that still provide marginal contribution.

\subsection{Ablation Study}
\noindent\textbf{Lifecycle Components.}
Table~\ref{tab:ablation_lifecycle} isolates the lifecycle operations. Removing retirement drops ALFWorld success from 87.5 to 73.4, showing that expansion without deletion degenerates toward SkillRL-like accumulation. Without expansion, \methodname{} makes no creator calls and reaches 78.9, exceeding SkillRL$^\dagger$ by 3.9 points. Expansion adds 8.6 points through new capability coverage, subsequent policy adaptation, and continued lifecycle management. Full \methodname{} preserves its advantage across creators, reaching 86.7 with frozen Qwen3-4B and 87.5 with o3 (Appendix~\ref{app:skill_creator_sensitivity}). Figure~\ref{fig:ablation_reward} shows the same pattern: local reward improvements can be transient, but the final endpoint remains worse without the full lifecycle.

\noindent\textbf{Effectiveness of External Marginal Contribution.}
``Random Audit'' keeps the same operation space but replaces contribution-aware decisions with stochastic ones: retain/delete is sampled with probabilities 0.8/0.2, and expansion is triggered independently with probability 0.1. It obtains only 68.8, the largest drop, and its reward curve stays below the other variants for most of training. Thus, the gain of \methodname{} does not come from random skill-set perturbation; lifecycle decisions must track whether a skill still provides marginal external value under the current policy.

\noindent\textbf{Beyond Prompt-Budget Control.}
``Fixed Active Set Size'' controls the active skill count at the initial size of 38, using LRU removal after expansion and expansion after retirement to keep the budget fixed. It reaches 75.6, still 11.9 points below \methodname{}. This rules out a pure prompt-budget explanation: the key is not only how many skills remain active, but which skills are retained, removed, and expanded.

\subsection{Case Study: Analysis of Skill Lifecycle Management}
\begin{figure*}[t]
\centering
\includegraphics[width=0.75\textwidth]{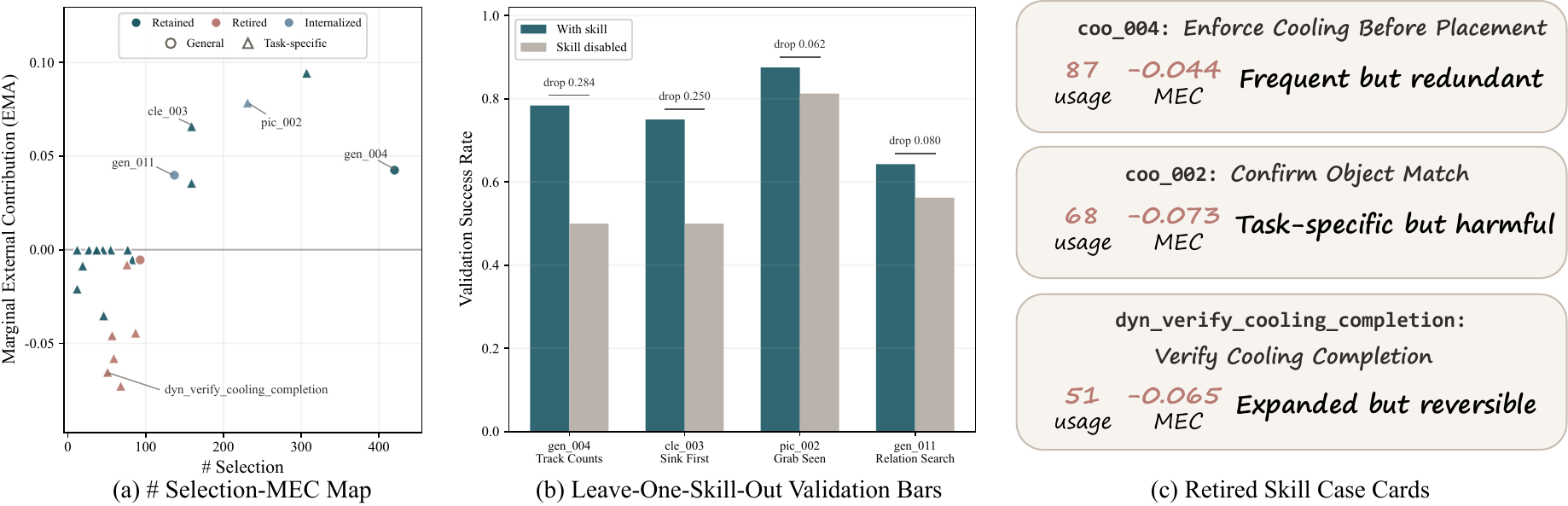}
\caption{Case study of skill lifecycle on ALFWorld. Panel (a) plots selection count against marginal external contribution (MEC) for retained, retired, and internalized skills. Panel (b) reports leave-one-skill-out validation bars for representative retained skills. Panel (c) shows retired skill cases.}
\vspace{-1.5em}
\label{fig:lifecycle_analysis}
\end{figure*}

We use a lifecycle probe (Appendix~\ref{sec:slim_setup}) to analyze \methodname{} operations. The probe logs each audited skill's selection count, marginal external contribution (MEC), with-skill performance, disabled-skill performance, and lifecycle outcome. Among frequently selected candidates, we label a skill as closer to \emph{internalized} when its with-skill validation success rate is at least 60\% and disabling it preserves at least 85\% of that performance. This diagnostic label describes policy dependence, while lifecycle updates follow Section~\ref{sec:lifecycle_rl}. While Figure~\ref{fig:training_dynamics} shows that the active set first expands, then contracts, and finally remains non-empty, Figure~\ref{fig:lifecycle_analysis} explains how contribution-aware filtering yields this compact set of externally valuable skills.

\noindent\textbf{Lifecycle Follows MEC.}
Panel (a) and (b) show that lifecycle decisions are governed by MEC rather than selection frequency alone. Broad and frequently selected skills such as \texttt{pic\_002} and \texttt{gen\_011} become close to internalized, as disabling them causes only 0.062 and 0.080 drops; however, high frequency is not sufficient for retirement, since \texttt{gen\_004} is also broad and frequent but remains externally valuable with a much larger 0.284 drop. Conversely, low frequency does not imply low value: \texttt{cle\_003} is less frequently selected but has high MEC, and disabling it causes a 0.250 drop, making it globally infrequent but locally indispensable. Panel (c) further shows that frequent, task-specific, or newly expanded skills can all be retired once their marginal external value becomes negative. Thus, \methodname{} retires a skill only when its external marginal value disappears under the current policy, while preserving long-tail skills that remain useful on specific routed subsets.

\noindent\textbf{Empirical Support for Local Lifecycle Signals.}
The case study also connects the local theory view to observed behavior. The MEC probe shows that leave-one-skill-out drops align with retain and retire outcomes, while Figure~\ref{fig:training_dynamics} shows that these local decisions produce a non-monotonic active set rather than monotone accumulation or elimination. The ablations further support this mechanism: w/o Retirement falls to 73.4, Random Audit to 68.8, and Fixed Active Set Size to 75.6, indicating that performance depends on contribution-aware lifecycle decisions rather than unbounded skill growth, random mutation, or simple prompt-budget control.

\section{Conclusion and Future Work}
We present \methodname{}, a framework for dynamic skill lifecycle management in agentic RL. The key idea of \methodname{} is that external skills should neither be assumed to accumulate indefinitely nor be forced to vanish toward zero-skill inference. Instead, the active external skill set should be dynamically updated as the policy evolves, allowing the learned endpoint to avoid both persistent full accumulation and forced zero-skill inference.
Across ALFWorld and SearchQA, \methodname{} improves over standard RL and prior skill-based methods while learning a qualitatively different endpoint. Training dynamics, ablations, and lifecycle case studies show that the active set evolves non-monotonically and that contribution-aware decisions are essential. These results suggest a division of labor between model parameters and external procedural memory: reusable behaviors can become less externally necessary, while narrow or locally consequential skills remain useful as external modules. Future work can extend this perspective to richer multimodal environments, finer-grained lifecycle units, and more scalable auditing methods.

\bibliography{ref}
\bibliographystyle{iclr2027_conference}

\clearpage
\appendix
\setcounter{table}{0} 
\setcounter{figure}{0}
\setcounter{equation}{0}
\renewcommand{\thetable}{A\arabic{table}}
\renewcommand\thefigure{A\arabic{figure}} 
\renewcommand\theequation{A\arabic{equation}}
The appendix is organized as below:\\
$\bullet$ Appendix~\ref{app:theory} presents the local theoretical analysis that complements \methodname{}.\\
$\bullet$ Appendix~\ref{sec:implementation_details} describes the detailed settings of \methodname{} and baselines.\\
$\bullet$ Appendix~\ref{app:setup} describes the detailed experimental setup, including training configuration, retrieval settings, and lifecycle auditing details.\\
$\bullet$ Appendix~\ref{app:additional_results} provides additional experiments on generalization, skill-bank transfer, training robustness, creator and hyperparameter sensitivity, and audit overhead.\\
$\bullet$ Appendix~\ref{app:prompts} introduces all prompts used in training, evaluation, and skill expansion.\\
$\bullet$ Appendix~\ref{app:skill_bank} summarizes the initial and expanded skill banks used by \methodname{}.\\
$\bullet$ Appendix~\ref{app:limitations} discusses the limitations of \methodname{}.\\

\section{Theoretical Analysis}
\label{app:theory}
This section formalizes the skill management problem in Section~\ref{sec:preliminaries} and analyzes the MEC-based lifecycle rules. We connect the motivation for persistent external support to knowledge representation, derive an error bound for smoothed MEC, and bound the performance loss of sequential retirements. The lifecycle analysis uses expected success rates on a fixed task distribution for validation and follows the policy and active set at each audited decision.

\subsection{Parametric Knowledge and External Support}
\label{app:representation_support}
The Physics of Language Models studies distinguish knowledge storage, extraction, and manipulation. Part~3.3 develops bit-complexity bounds and studies how knowledge storage depends on model size and training exposure in controlled factual-knowledge settings~\cite{Zhu20224physics33}. Parts~3.1 and~3.2 show that storing knowledge and reliably extracting or manipulating it are distinct requirements~\cite{Zhu20224physics31,Zhu20224physics32}. This distinction motivates evaluating an external skill through its contribution to the current policy's behavior. A simple counting argument makes the representation issue precise.

\begin{proposition}[Finite Representation]
\label{prop:finite_representation}
Consider a controlled knowledge-recovery problem with fixed background information. Let $Z$ be uniformly distributed over $N$ target knowledge states, and let a training procedure encode $Z$ into a parameter state $W$ taking values in a finite set $\mathcal{W}$. A decoder that receives only $W$ and the fixed background produces $\widehat Z$, possibly using independent randomness. Then
\begin{equation}
\Pr(\widehat Z=Z)\le \min\!\left\{1,\frac{|\mathcal{W}|}{N}\right\}.
\label{eq:finite_representation}
\end{equation}
Consequently, recovery with error at most $\epsilon\in[0,1)$ requires $|\mathcal{W}|\ge(1-\epsilon)N$.
\end{proposition}
\begin{proof}
Write $p(w\mid z)$ for the encoder distribution and $d(z\mid w)$ for the decoder distribution. Since the decoder accesses the target state only through $W$,
\begin{align}
\Pr(\widehat Z=Z)
&=\frac{1}{N}\sum_{w\in\mathcal W}\sum_{z=1}^{N}p(w\mid z)d(z\mid w)\nonumber\\
&\le\frac{1}{N}\sum_{w\in\mathcal W}\sum_{z=1}^{N}d(z\mid w)
\le\frac{|\mathcal W|}{N}.
\end{align}
The probability is also at most one. Rearranging $1-\epsilon\le|\mathcal W|/N$ gives the necessary condition.
\end{proof}

Finite-precision parameterizations have finitely many states, so this observation applies to their representation of a controlled knowledge family. External artifacts provide an additional input through which task-relevant information can be supplied. Together with the distinction between storage and usable knowledge, this motivates a behavioral criterion: external support remains useful when disabling it reduces task performance. The following analysis quantifies this criterion directly through MEC.

\subsection{Task-Conditioned Contribution}
\label{app:contribution_analysis}
Let $\mathcal D$ be a fixed distribution over tasks used for validation, such as the uniform empirical distribution on the development split. For a policy $\pi_\theta$ and active set $\mathcal A$, define
\begin{equation}
P_\theta(x,\mathcal A)
:=\mathbb E\!\left[\operatorname{Perf}(x;\pi_\theta,\mathcal A)\right]\in[0,1],
\qquad
F_{\mathcal D}(\theta,\mathcal A)
:=\mathbb E_{x\sim\mathcal D}[P_\theta(x,\mathcal A)],
\label{eq:audit_performance}
\end{equation}
where the first expectation includes rollout randomness. Thus the bounds below concern performance on the specified audit distribution; test split performance is evaluated by the experiments.

Fix an audit after the policy update to $\theta_{t+1}$, and let $\mathcal A_t$ denote the active set immediately before the skill comparison. Fix the retrieval rule and its tie-breaking. For an audited skill $s\in\mathcal A_t$, define the pre-ablation routing event and its probability by
\begin{equation}
R_t(s)=\{x:s\in\mathcal A_t^{\mathrm{gen}}\cup\mathcal Q_t(x)\},
\qquad q_t(s)=\Pr_{x\sim\mathcal D}(R_t(s))>0.
\end{equation}
The routed expected MEC is
\begin{equation}
\Delta_t^{\mathrm{loc}}(s)
:=\mathbb E_{x\sim\mathcal D}\!\left[
P_{\theta_{t+1}}(x,\mathcal A_t)
-P_{\theta_{t+1}}(x,\mathcal A_t\setminus\{s\})
\mid R_t(s)\right].
\label{eq:local_population_mec}
\end{equation}
The routing event is held fixed for the paired comparison, while the counterfactual rollout retrieves from the remaining active set. Therefore, the difference includes the effect of any replacement skill. It conditions on all other active skills and allows their contributions to interact.

\begin{lemma}[From Routed to Aggregate Contribution]
\label{lem:routed_contribution}
If disabling $s$ leaves expected performance unchanged outside its pre-ablation routing event,
\begin{equation}
P_{\theta_{t+1}}(x,\mathcal A_t)
=P_{\theta_{t+1}}(x,\mathcal A_t\setminus\{s\})
\quad\text{for $\mathcal D$-almost every $x\notin R_t(s)$},
\label{eq:off_route_invariance}
\end{equation}
then
\begin{equation}
F_{\mathcal D}(\theta_{t+1},\mathcal A_t)
-F_{\mathcal D}(\theta_{t+1},\mathcal A_t\setminus\{s\})
=q_t(s)\Delta_t^{\mathrm{loc}}(s).
\label{eq:local_global_mec}
\end{equation}
\end{lemma}
\begin{proof}
Split the expected performance difference over $R_t(s)$ and its complement. The first term is $q_t(s)\Delta_t^{\mathrm{loc}}(s)$ by definition, and the second term is zero by Eq.~\eqref{eq:off_route_invariance}.
\end{proof}

Equation~\eqref{eq:off_route_invariance} holds when removing an unselected skill leaves the retrieved context and rollout protocol unchanged. Deterministic top-$K$ retrieval over fixed skill scores has this property for skills outside the selected set. This condition makes explicit when the controller's local comparison also determines its aggregate effect on $\mathcal D$.

\subsection{Audit Error and EMA Tracking}
\label{app:audit_error}
For a fixed skill, index its successive audits by $j=1,\ldots,J$. Write $\widehat\Delta_j$ for the raw estimate in Eq.~\eqref{eq:marginal_contribution}, $\mu_j$ for its current routed expected MEC, and $\overline\Delta_j$ for its EMA. The policy, active set, and routed task subset may change between these audits.

\begin{assumption}[Audit Accuracy and Contribution Drift]
\label{ass:validation_estimation}
For a finite audit horizon, choose error budgets $\delta_j\in(0,1)$ before the comparisons, with total budget below one. Conditional on the history $\mathcal H_{j-1}$ that selects the policy, active set, and skill for comparison $j$, the raw estimate satisfies
\begin{equation}
\Pr\!\left(|\widehat\Delta_j-\mu_j|>\sigma_j\mid\mathcal H_{j-1}\right)
\le\delta_j.
\label{eq:raw_audit_accuracy}
\end{equation}
The change in routed expected contribution satisfies $|\mu_j-\mu_{j-1}|\le d_j$, and the initial estimate satisfies $|\overline\Delta_0-\mu_0|\le e_0$. The nonnegative bounds $\sigma_j,d_j,e_0$ describe estimation error, contribution drift, and initialization error, respectively.
\end{assumption}

A sufficient sampling condition for Eq.~\eqref{eq:raw_audit_accuracy} is that, conditional on the selected comparison, $\widehat\Delta_j$ averages $n_j$ independent paired success differences $Y_{j,r}\in[-1,1]$ from the routed task distribution. Fresh task and rollout draws, or a complete independently rolled-out pass over the corresponding empirical task set, give average conditional mean $\mu_j$. Hoeffding's inequality~\cite{hoeffding1963probability} then gives
\begin{equation}
\Pr\!\left(|\widehat\Delta_j-\mu_j|>z\mid\mathcal H_{j-1}\right)
\le 2\exp(-n_jz^2/2),
\qquad
\sigma_j=\sqrt{\frac{2\log(2/\delta_j)}{n_j}}.
\label{eq:paired_audit_concentration}
\end{equation}
Here $n_j$ counts independent evaluation pairs for that comparison; the cumulative exposure $u_t(s)$ is a separate eligibility statistic. For reused rollouts, the required accuracy bound concerns the resulting joint sampling scheme. The conditional formulation in Assumption~\ref{ass:validation_estimation} accounts for adaptive candidate selection.

\begin{lemma}[Tracking the Current MEC]
\label{lem:retention_proxy}
Under Assumption~\ref{ass:validation_estimation}, let $0<\alpha\le1$ and use the EMA update from Section~\ref{sec:marginal_contribution}. Define
\begin{equation}
e_j=(1-\alpha)(e_{j-1}+d_j)+\alpha\sigma_j.
\label{eq:ema_error_recursion}
\end{equation}
With probability at least $1-\sum_{j=1}^{J}\delta_j$, all $J$ estimates satisfy $|\overline\Delta_j-\mu_j|\le e_j$. If $\sigma_j\le\sigma$ and $d_j\le d$, then
\begin{equation}
e_j\le (1-\alpha)^j e_0
+\bigl(1-(1-\alpha)^j\bigr)
\left(\sigma+\frac{1-\alpha}{\alpha}d\right).
\label{eq:ema_error_bound}
\end{equation}
\end{lemma}
\begin{proof}
Taking expectations over the history in Eq.~\eqref{eq:raw_audit_accuracy} and applying a union bound gives an event of probability at least $1-\sum_j\delta_j$ on which every raw error is at most $\sigma_j$. On this event,
\begin{align}
\overline\Delta_j-\mu_j
&=(1-\alpha)(\overline\Delta_{j-1}-\mu_{j-1})
 +(1-\alpha)(\mu_{j-1}-\mu_j)
 +\alpha(\widehat\Delta_j-\mu_j),\nonumber\\
|\overline\Delta_j-\mu_j|
&\le (1-\alpha)(e_{j-1}+d_j)+\alpha\sigma_j=e_j.
\end{align}
Induction establishes the simultaneous bound. Summing the geometric recursion with $\sigma_j\le\sigma$ and $d_j\le d$ yields Eq.~\eqref{eq:ema_error_bound}.
\end{proof}

The error bound separates sampling noise from the effect of earlier policy and skill-set states. Refreshing a raw MEC estimate after a retirement updates the comparison to the new active set; the drift term accounts for older information retained by the EMA. No independence across smoothed estimates is needed for the simultaneous bound. Minimum exposure and patience add eligibility requirements to the retirement rule, so the bounds below apply to the subset of decisions that satisfy those requirements.

\subsection{Retention and Sequential Retirement}
\label{app:retirement_analysis}
At a particular audited decision, abbreviate the error bound from Lemma~\ref{lem:retention_proxy} as $\varepsilon_t(s)$, so that $|\bar\Delta_t(s)-\Delta_t^{\mathrm{loc}}(s)|\le\varepsilon_t(s)$ on its accuracy event.

\begin{lemma}[Protection of Useful Skills]
\label{lem:useful_skill_protection}
On the audit accuracy event, any currently audited active skill satisfying
\begin{equation}
\Delta_t^{\mathrm{loc}}(s)\ge\tau_{\mathrm{retire}}+\varepsilon_t(s)
\end{equation}
cannot be retired by Eq.~\eqref{eq:retire_rule}. If Eq.~\eqref{eq:retain_rule} is triggered, its routed expected contribution is at least $\tau_{\mathrm{keep}}-\varepsilon_t(s)$.
\end{lemma}
\begin{proof}
The first condition implies $\bar\Delta_t(s)\ge\Delta_t^{\mathrm{loc}}(s)-\varepsilon_t(s)\ge\tau_{\mathrm{retire}}$, which prevents the strict retirement inequality. For the retain rule, $\bar\Delta_t(s)\ge\tau_{\mathrm{keep}}$ implies $\Delta_t^{\mathrm{loc}}(s)\ge\bar\Delta_t(s)-\varepsilon_t(s)\ge\tau_{\mathrm{keep}}-\varepsilon_t(s)$.
\end{proof}

\begin{lemma}[Single-Retirement Loss]
\label{lem:retirement_loss}
On the audit accuracy event, if Eq.~\eqref{eq:retire_rule} removes $s$, then
\begin{equation}
\Delta_t^{\mathrm{loc}}(s)<\tau_{\mathrm{retire}}+\varepsilon_t(s).
\end{equation}
If Eq.~\eqref{eq:off_route_invariance} also holds, the aggregate performance change satisfies
\begin{equation}
F_{\mathcal D}(\theta_{t+1},\mathcal A_t\setminus\{s\})
-F_{\mathcal D}(\theta_{t+1},\mathcal A_t)
\ge -q_t(s)\bigl(\tau_{\mathrm{retire}}+\varepsilon_t(s)\bigr).
\label{eq:single_retirement_loss}
\end{equation}
\end{lemma}
\begin{proof}
Retirement requires $\bar\Delta_t(s)<\tau_{\mathrm{retire}}$. Therefore $\Delta_t^{\mathrm{loc}}(s)\le\bar\Delta_t(s)+\varepsilon_t(s)<\tau_{\mathrm{retire}}+\varepsilon_t(s)$. Lemma~\ref{lem:routed_contribution} expresses the aggregate change as $-q_t(s)\Delta_t^{\mathrm{loc}}(s)$, giving the bound.
\end{proof}

\begin{theorem}[Sequential Retirement]
\label{thm:sequential_retirement}
Fix the policy $\theta$ and audit distribution $\mathcal D$ during one lifecycle audit. Suppose the committed retirements are $s_1,\ldots,s_m$, with
\begin{equation}
\mathcal A^{(i)}=\mathcal A^{(i-1)}\setminus\{s_i\},
\qquad i=1,\ldots,m.
\end{equation}
For every removal, compute the routing probability $q_i$ and the current MEC relative to $\mathcal A^{(i-1)}$, and let $\varepsilon_i$ bound the corresponding smoothed estimation error. If Eq.~\eqref{eq:off_route_invariance} holds for each removal, then, on the simultaneous accuracy event,
\begin{equation}
F_{\mathcal D}(\theta,\mathcal A^{(m)})
\ge F_{\mathcal D}(\theta,\mathcal A^{(0)})
-\sum_{i=1}^{m}q_i\bigl(\tau_{\mathrm{retire}}+\varepsilon_i\bigr).
\label{eq:sequential_retirement_loss}
\end{equation}
The event has probability at least $1-\sum_{a\in\mathcal H}\delta_a$ when Assumption~\ref{ass:validation_estimation} holds over a finite audit horizon $\mathcal H$ covering the raw estimates used in these skills' EMA histories, with valid initialization and drift bounds.
\end{theorem}
\begin{proof}
Apply Lemma~\ref{lem:retirement_loss} to each actual removal from its immediately preceding active set:
\begin{equation}
F_{\mathcal D}(\theta,\mathcal A^{(i)})
-F_{\mathcal D}(\theta,\mathcal A^{(i-1)})
\ge -q_i\bigl(\tau_{\mathrm{retire}}+\varepsilon_i\bigr).
\end{equation}
Summing over $i$ telescopes the left-hand side and proves Eq.~\eqref{eq:sequential_retirement_loss}. Taking a union bound over the raw accuracy events in $\mathcal H$ and applying the EMA recursion gives the stated probability. The argument holds for the realized adaptive sequence, including its stopping point, on this simultaneous event.
\end{proof}

Each term in Eq.~\eqref{eq:sequential_retirement_loss} uses the active set produced by all earlier removals, so the result allows interactions among skills. For example, suppose two skills provide substitutable support. After the first is removed, if the second has MEC at least $\tau_{\mathrm{retire}}+\varepsilon_i$ under the updated active set, Lemma~\ref{lem:useful_skill_protection} prevents its subsequent retirement. Reusing both skills' pre-removal MEC values would not establish this conclusion. The cumulative bound pertains to the chosen audit order and the evaluated task distribution.

\subsection{Policy Learning and Skill Expansion}
\label{app:lifecycle_cycle}
GRPO changes the policy, while expansion changes the support available for future rollouts. To relate these stages to the retirement analysis, let $\mathcal A_t^{-}$ be the set after retirements and $\mathcal A_{t+1}$ the set after all accepted skill creations in the cycle. Define the signed performance changes
\begin{align}
g_t^{\mathrm{policy}}
&=F_{\mathcal D}(\theta_{t+1},\mathcal A_t)
 -F_{\mathcal D}(\theta_t,\mathcal A_t),\nonumber\\
g_t^{\mathrm{expand}}
&=F_{\mathcal D}(\theta_{t+1},\mathcal A_{t+1})
 -F_{\mathcal D}(\theta_{t+1},\mathcal A_t^{-}).
\end{align}
Here $g_t^{\mathrm{expand}}$ is the joint effect of the additions, including changes in retrieval. Adding and subtracting the intermediate performances and applying Theorem~\ref{thm:sequential_retirement} yields
\begin{equation}
\begin{aligned}
F_{\mathcal D}(\theta_{t+1},\mathcal A_{t+1})
-F_{\mathcal D}(\theta_t,\mathcal A_t)
\ge g_t^{\mathrm{policy}}+g_t^{\mathrm{expand}}
-\sum_{i=1}^{m}q_i\bigl(\tau_{\mathrm{retire}}+\varepsilon_i\bigr).
\end{aligned}
\label{eq:lifecycle_cycle_bound}
\end{equation}
Thus, a cycle improves audit-distribution performance whenever the policy and expansion gains exceed the retirement loss bound. Equation~\eqref{eq:expand_rule} uses persistent routed failures to propose new skills; their contribution is assessed by subsequent audits. The quantities $g_t^{\mathrm{policy}}$ and $g_t^{\mathrm{expand}}$ describe performance changes in the analysis, while the implemented decisions use the GRPO objective, MEC, and failure statistics.

The resulting interpretation is behavioral: learning changes which external skills the policy benefits from, and lifecycle auditing tracks that changing dependence. Lemma~\ref{lem:useful_skill_protection} preserves support whose current contribution exceeds the estimation margin; Theorem~\ref{thm:sequential_retirement} controls the loss of removing low-contribution support. These properties apply whether a low contribution arises from policy learning, substitutable skills, or ineffective guidance. The internalization probe in Appendix~\ref{sec:slim_setup} separately examines the policy's performance with and without particular skills.

\section{Implementation Details}
\label{sec:implementation_details}
\subsection{\methodname{} Setup}
\label{sec:slim_setup}

\noindent\textbf{Backbone and Data Protocol.}
All main \methodname{} experiments use Qwen3-4B~\cite{qwen3technical} as the policy model. We use the train split for GRPO updates, the development split for lifecycle auditing and training-time monitoring, and the test split only for final reporting. For ALFWorld, we use 16 tasks from the train split per update, 32 tasks from the development split per validation pass, a maximum prompt length of 4096, a maximum response length of 512, right truncation, and no overlength prompt filtering. For SearchQA, we use 64 tasks from the train split per update, 512 tasks from the development split per validation pass, a maximum prompt length of 5000, a maximum response length of 700, left truncation, and overlength prompt filtering. ALFWorld runs for 120 GRPO steps and SearchQA runs for 180 GRPO steps both with validation every 10 steps. Validation and final evaluation use sampled generation with temperature 0.4.

\noindent\textbf{Training.}
\methodname{} uses the GRPO objective Eq.~\eqref{eq:grpo_obj} and does not use a separate warmup or cold-start stage. ALFWorld uses $n=8$ rollouts per prompt, a maximum of 50 environment steps, and 50 turns of interaction history. SearchQA uses $n=4$ rollouts per prompt, a maximum of 4 environment steps, 4 turns of history, and the shared search backend used by all SearchQA methods. The optimizer uses learning rate $10^{-6}$ in both benchmarks. The PPO mini-batch and per-device micro-batch sizes are 32 and 2 for ALFWorld, and 512 and 8 for SearchQA. Rewards are outcome-level environment rewards: completed successful trajectories receive the benchmark success reward and failed trajectories receive zero reward, while invalid-action penalties are applied during trajectory collection. The invalid-action penalty coefficient is 0.1 for ALFWorld and 0.01 for SearchQA. For the main \methodname{} runs, both actor-side KL loss and KL-in-reward regularization are disabled.

\noindent\textbf{Retrieval.}
\methodname{} uses the hierarchical SkillRL-style skill bank described in Section~\ref{sec:preliminaries}. Task-specific skills are retrieved only from the active pool of the detected task type. The retrieval query is the current task description, while each skill key is a routing text concatenating the skill title, description or principle, \texttt{when\_to\_apply} field, body, tags, and task type. We embed both query and keys with Qwen3-Embedding-0.6B~\cite{zhang2025qwen3embedding} and rank candidate skills by cosine similarity. We set the task-specific retrieval cap to $K=3$ and the embedding threshold to $\tau_{\mathrm{emb}}=0.45$, so retrieval may insert fewer than three task-specific skills when no active skill is sufficiently relevant.

\noindent\textbf{Skill Lifecycle.}
Lifecycle auditing runs periodically after GRPO validation, recording routed skills, validation outcomes, and routed failures under the current active set. With the policy fixed, candidates are audited sequentially. ALFWorld audits at most four recently routed task-specific skills from the top-$K$ retrieved set, while SearchQA audits at most 12 active skills, supported by its larger validation batches and shorter episodes. Missing candidates are skipped, keeping each audit within its budget. For each candidate $s$, we record the performance of the current active set on its routed subset of the development split, temporarily disable $s$, and rerun the same tasks to compute the leave-one-skill-out marginal external contribution $\Delta_t(s)$. We then update the smoothed estimate $\bar{\Delta}_t(s)$ with EMA coefficient 0.9 and apply the retain and retire criteria. Retain decisions use $\tau_{\mathrm{keep}}=0.03$ for ALFWorld and $\tau_{\mathrm{keep}}=0.05$ for SearchQA. Retirement uses $\tau_{\mathrm{retire}}=0.001$ and patience $p=3$ in both benchmarks, together with minimum exposure thresholds of $n_{\min}=30$ for ALFWorld and $n_{\min}=20$ for SearchQA. Thus, a skill is retired after sufficient routed exposure and persistently negligible contribution. Each triggered retirement is committed immediately, followed by refreshed skill routing and a recomputed full-active-set validation baseline for the next candidate. Subsequent MEC estimates reflect the updated active set. For two substitutable skills that jointly provide necessary coverage, retiring one removes their redundancy and allows the next audit to reveal the value of the remaining skill.

\noindent\textbf{Skill Expansion.}
Expansion operates during training audits and adds task-specific skills from validation failure buckets, each grouping failures routed to the same active skill and task type. We select the expansion threshold together with the retention and retirement thresholds using performance on a smaller subset of the train split and keep them fixed across independent training runs of each setting. Both benchmarks use $\tau_{\mathrm{expand}}=0.40$. For each qualifying audited failure bucket, ALFWorld uses $n_{\mathrm{expand}}=20$ and creates at most $B=2$ new skills, while SearchQA uses $n_{\mathrm{expand}}=15$ and creates at most $B=3$. The sensitivity study in Appendix~\ref{app:hyperparameter_sensitivity} varies these three lifecycle thresholds one at a time on ALFWorld, with minimum exposure, patience, and the other training settings held fixed. For each selected failure bucket, an OpenAI-compatible \texttt{o3} skill-creator backbone generates standalone \texttt{SKILL.md} artifacts using an Anthropic-style skill-creator workflow~\cite{anthropic2025skills}. The expansion prompt in Figure~\ref{fig:skill_creation_prompt} includes representative failed tasks, summarized failure traces, and the insufficient active skills. We deduplicate against existing skills by title, trigger description, and embedding similarity of the routing text, and reject generated skills that are generic, duplicate an existing skill, are too short to define a useful workflow, or attempt to create new general or foundational skills. Accepted skills inherit the bucket's task type and enter the corresponding task-specific pool.

\noindent\textbf{Audit Cost.}
Leave-one-skill-out auditing adds validation-time overhead, but it is bounded by the audit schedule and candidate budget. Audits run every 10 GRPO steps and use limited subsets of the development split. In ALFWorld, each audit evaluates at most one general-skill group and $M=4$ task-specific skills; in SearchQA, the larger validation batch and shorter episodes allow at most 12 audited active skills. Each committed retirement additionally refreshes routing and the validation baseline for the remaining candidates. These refreshes are bounded by the number of audited candidates, keeping the validation cost linear in the audit budget. In our runs, ALFWorld completes in about 20 hours and SearchQA completes in about 25 hours, which is the same order as Skill0 and SkillRL under the shared training stack. Table~\ref{tab:audit_overhead} in Appendix~\ref{app:audit_overhead} provides the detailed overhead comparison.

\noindent\textbf{Inference.}
For skill-conditioned final inference, \methodname{} uses the final active set after training. Before each rollout, the agent inserts active general skills and retrieves task-specific skills using the same task-type-scoped embedding retrieval as in training. Final evaluation does not perform retain, retire, or expand operations. The no-skill evaluation uses the same trained policy with both general and task-specific skill insertion disabled at inference. The environment prompts are shown in Figures~\ref{fig:alfworld_skill_prompt} and~\ref{fig:searchqa_skill_prompt}, and the skill insertion format is shown in Figure~\ref{fig:skill_insertion_prompt}.

\noindent\textbf{Internalization Probe.}
For the ALFWorld lifecycle case study in Figure~\ref{fig:lifecycle_analysis}, we run a diagnostic probe every 20 training steps. Each probe selects up to five retained skills with at least 10 validation selections and utility EMA of at least 0.003. For each candidate $s$, $\operatorname{Perf}_{\mathrm{with}}$ is the success rate on its routed subset of the development split $\mathcal{V}_t(s)$ under the current policy and active set $\mathcal{A}_t$. We then rerun the same tasks with the policy fixed and only $s$ disabled to obtain $\operatorname{Perf}_{\mathrm{without}}$. We label a skill as closer to internalized when $\operatorname{Perf}_{\mathrm{with}}\ge0.60$ and $\operatorname{Perf}_{\mathrm{without}}/\operatorname{Perf}_{\mathrm{with}}\ge0.85$. The label identifies skills for which the policy preserves at least 85\% of its skill-conditioned performance after the skill is disabled. Probe outcomes are recorded for the case study, while training-time retain, retire, and expand decisions follow the lifecycle rules above.

\noindent\textbf{Compute Resources.}
From logged token counts and step timings, the ALFWorld full \methodname{} run is on the order of $10^{19}$ floating-point operations. SearchQA uses shorter episodes but a larger validation batch and longer training horizon, giving an overall compute budget on the order of $10^{19}$--$10^{20}$ floating-point operations. These estimates include rollout generation, log-probability evaluation, policy updates, validation, and lifecycle audit reruns, but exclude one-time data preprocessing.

\subsection{Baselines Setup}
\label{sec:baseline_setup}
We use the original implementations of ReAct, Reflexion, ExpeL, Mem0, EvolveR, Skill0, and SkillRL with the benchmark configurations specified below.

\noindent\textbf{Zero-Shot and Few-Shot Prompting.}
The prompt-only baselines evaluate the backbone model without RL updates. Zero-shot prompting directly uses the benchmark interaction prompt in Figure~\ref{fig:alfworld_base_prompt} or Figure~\ref{fig:searchqa_base_prompt}. Few-shot prompting prepends solved examples following Figure~\ref{fig:fewshot_prompt}. The with-skill variants use the same prompt protocol but additionally insert retrieved external skills using Figure~\ref{fig:skill_insertion_prompt}. Prompt baselines are evaluated with the same environment parser, success metric, and test split as the RL methods, using batch size 16 for ALFWorld and 32 for SearchQA.

\noindent\textbf{ReAct.}
ReAct interleaves reasoning and actions to solve tasks. It uses the instruction in Figure~\ref{fig:react_prompt}, preserves the same action tags and parser as the main environment, and does not use persistent memory.

\noindent\textbf{Reflexion.}
Reflexion stores short reflections from failed trajectories and retrieves relevant reflections before later actions. The acting and reflection-generation prompts are shown in Figure~\ref{fig:reflexion_prompt}. We retrieve at most three reflections and keep a bounded memory of 200 items. Generation uses deterministic service decoding with maximum 768 tokens, temperature 0, and top-$p=1.0$.

\noindent\textbf{ExpeL.}
ExpeL learns reusable lessons from past experience. After each episode, the method distills a compact reusable lesson from the trajectory and retrieves at most three relevant lessons for future decisions, with the same 200-item memory cap and deterministic service decoding as Reflexion. Figure~\ref{fig:expel_prompt} gives the acting and lesson-distillation prompts. Evaluation uses the same action parser, success metric, and protocol for the train split, development split, and test split as the other baselines.

\noindent\textbf{Mem0.}
Mem0 maintains long-term memory across tasks. It extracts up to three short atomic memories after each episode, stores them in the same bounded 200-item memory bank, and retrieves at most three relevant memories before action generation, following Figure~\ref{fig:mem0_prompt}.

\noindent\textbf{GRPO and GRPO with Skills.}
The GRPO baseline uses the same RL optimizer, environment stack, and reward definition as \methodname{}, but does not receive external skill or memory context and does not use a warmup stage. ALFWorld uses train batch size 16, validation batch size 32, rollout count 8, maximum episode length 50, learning rate $10^{-6}$, PPO mini-batch size 32, micro-batch size 2, 120 training steps, and validation every 10 steps. SearchQA uses train batch size 64, validation batch size 512, rollout count 4, maximum episode length 4, learning rate $10^{-6}$, PPO mini-batch size 512, micro-batch size 8, 180 training steps, and validation every 10 steps. Both settings disable actor-side KL loss and KL-in-reward regularization, and use invalid-action penalty coefficients 0.1 for ALFWorld and 0.01 for SearchQA. GRPO with skills keeps the same optimizer and rollout settings but inserts retrieved external skills using the skill-conditioned prompt format in Figure~\ref{fig:skill_insertion_prompt}; it does not perform retain, retire, or expand operations.

\noindent\textbf{EvolveR.}
EvolveR is evaluated as an experience-lifecycle retrieval baseline. It distills past trajectories into reusable principles, retrieves relevant experience before rollout, and updates the experience store during training. We evaluate EvolveR on ALFWorld and SearchQA using the shared success-rate metric and protocol for the train split, development split, and test split. Where experience distillation or update requires a service model, we use the same OpenAI-compatible service family as the other agent and memory baselines. EvolveR retrieves the top-3 experience principles before each rollout. It does not estimate leave-one-skill-out marginal external contribution and does not maintain retain/retire/expand lifecycle states over skill artifacts.

\noindent\textbf{Skill0.}
Skill0 follows its helpfulness-driven dynamic curriculum without a separate warmup stage in our comparison. Every $d$ steps, Skill0 evaluates each skill file on its matched sub-task from the development split, computes its on-policy helpfulness as the performance difference between evaluation with and without that skill, filters out skills with non-positive helpfulness, ranks the remainder by helpfulness, and selects the top files under the current stage budget. In this Filter--Rank--Select procedure, helpfulness ranks skills within a predefined shrinking curriculum toward zero-skill inference. \methodname{} instead estimates smoothed MEC on each skill's currently routed subset of the development split and uses contribution, exposure, and failure signals to retain, retire, or expand a non-monotonic active set. ALFWorld uses the schedule $[6,3,0]$ with train batch size 16, validation batch size 32, rollout count 8, maximum episode length 50, learning rate $10^{-6}$, PPO mini-batch size 32, micro-batch size 2, 120 training steps, and validation every 10 steps. SearchQA uses the schedule $[5,3,0]$ with train batch size 64, validation batch size 512, rollout count 4, maximum episode length 4, learning rate $10^{-6}$, PPO mini-batch size 512, micro-batch size 8, and 180 training steps. Both settings disable actor-side KL loss and KL-in-reward regularization, and use the same invalid-action penalty coefficients as \methodname{}. Skill0 runs in the text-native ALFWorld and SearchQA environments with OCR and text rendering disabled.

\noindent\textbf{SkillRL.}
SkillRL uses persistent skill-augmented RL and runs without a separate warmup stage in our comparison. The policy retrieves external skills during training and inference with the same SkillRL-style retrieval path used to initialize our skill bank, using the skill-conditioned prompt format in Figure~\ref{fig:skill_insertion_prompt}. ALFWorld uses top-6 skill retrieval, dynamic skill-bank update with threshold 0.4, at most three new skills per update, train batch size 16, validation batch size 32, rollout count 8, maximum episode length 50, learning rate $10^{-6}$, PPO mini-batch size 32, micro-batch size 2, 120 training steps, and validation every 10 steps. Its skill author/updater uses the OpenAI-compatible \texttt{o3} backbone that analyzes failed trajectories and emits new JSON skill records, matching the SkillRL-style skill-bank format rather than producing \texttt{SKILL.md} artifacts. SearchQA uses the same top-6 skill retrieval and dynamic skill-bank update protocol, with train batch size 64, validation batch size 512, rollout count 4, maximum episode length 4, learning rate $10^{-6}$, PPO mini-batch size 512, micro-batch size 8, and 180 training steps. Unlike GRPO, SkillRL retains a reference-policy KL loss, using KL coefficients 0.01 for ALFWorld and 0.001 for SearchQA, while KL-in-reward remains disabled. Validation uses the dedicated development split, and the test split is reserved for final reporting.

\subsection{Additional Backbone and Benchmark Setup}
\label{sec:cross_backbone_setup}
The cross-backbone experiments train \methodname{}, SkillRL, and Skill0 independently with Qwen3-4B and Qwen3-8B on each environment. Each run starts from the corresponding base policy checkpoint and learns its own policy and skill-set trajectory. ALFWorld and SearchQA use their respective benchmark configurations described above and in Appendix~\ref{sec:baseline_setup}.

For WebShop, training runs for 150 GRPO steps with a maximum of 15 environment actions per episode. The remaining RL settings follow the ALFWorld configuration of each method. \methodname{} starts with 57 active skills and applies task-conditioned retrieval and lifecycle auditing during training. The train split, development split, and test split contain separate product-shopping goals. Final evaluation reports binary success rate on goals in the test split using the final policy and skill set.

\subsection{Additional Baselines Setup}
\label{sec:addtional_baseline_setup}
\noindent\textbf{GPT-4o and Gemini-2.5-Pro on ALFWorld.}
For the closed-source ALFWorld comparison, we report the GPT-4o and Gemini-2.5-Pro results from SkillRL~\cite{xia2026skillrl}. These results are used only in the expanded ALFWorld table in Appendix~\ref{app:expanded_results} to contextualize the scale of agent performance under strong proprietary models. They are not used for hyperparameter selection, lifecycle auditing, or any training-time comparison.

\noindent\textbf{SimpleMem on ALFWorld.}
SimpleMem is evaluated as a long-term memory baseline following its semantic memory compression and retrieval design~\cite{simplemem2026}. Each completed ALFWorld episode is converted into a compact textual memory record containing the task description, key observations, actions, outcome, and a short lesson. Memories are indexed with Qwen3-Embedding-0.6B to match the embedding backbone used elsewhere in our experiments. Following the original SimpleMem configuration style, we enable planning and reflection during retrieval, use semantic retrieval as the primary route, and allow keyword and structured filters as auxiliary routes. The memory retriever considers up to 25 semantic candidates, 5 keyword candidates, and 5 structured candidates, then inserts at most three compact memories into the action prompt to keep the prompt budget comparable to the other memory baselines. The policy model is Qwen3-4B, and the environment parser, success metric, train split, development split, test split, and final evaluation protocol are the same as the other baselines.

\noindent\textbf{RLOO on ALFWorld.}
RLOO is evaluated as an RL optimizer baseline using the same ALFWorld environment stack as GRPO. The only optimizer-level change is the advantage estimator: for each prompt, we sample $K=8$ rollouts and compute the leave-one-out advantage $\hat{A}_i=R_i-\frac{1}{K-1}\sum_{j\neq i}R_j$, following the REINFORCE leave-one-out formulation~\cite{ahmadian2024back}. All other settings are matched to the ALFWorld GRPO baseline, including train batch size 16, validation batch size 32, maximum episode length 50, learning rate $10^{-6}$, 120 training steps, validation every 10 steps, and no cold-start SFT. RLOO does not receive external skills, memory context, or lifecycle updates.

\noindent\textbf{MemRL on ALFWorld.}
MemRL is evaluated as a runtime memory-reinforcement baseline that updates episodic memory without updating model parameters~\cite{memrl2026}. We preserve the original method structure by using proceduralization for memory construction, query-based retrieval, and adjustment-based memory updates. To keep the comparison fair, the policy model is Qwen3-4B and the embedding model is Qwen3-Embedding-0.6B, while the ALFWorld parser, few-shot examples, test split, and success metric are shared with the rest of our evaluation. The memory retriever uses $k_{\mathrm{retrieve}}=5$, at most 8 extracted keywords, add-similarity threshold 0.90, novelty threshold 0.85, unknown-detection threshold 0.62, and a value-aware candidate set of size 3. The value update follows the original single-step setting with learning rate $\alpha=0.3$, discount $\gamma=0$, success reward 1.0, failure reward $-1.0$, and equal weights 0.5/0.5 for similarity and memory value in the combined retrieval score.

\noindent\textbf{RAG on SearchQA.}
RAG is evaluated as a one-shot retrieval-augmented generation baseline on SearchQA~\cite{lewis2020retrieval}. Following the Search-R1 setting~\cite{Jin2025SearchR1TL}, each question is used as the retrieval query, the shared search corpus is queried before generation, and the top three retrieved passages are inserted into the prompt. The model then generates a final answer without iterative search calls, RL updates, skill retrieval, or lifecycle control. This baseline uses the same Qwen3-4B backbone, answer-matching rule, and test split as the other SearchQA methods.

\noindent\textbf{Search-o1 on SearchQA.}
Search-o1 is evaluated as an inference-time agentic search baseline~\cite{li2025search}. The model is prompted to reason, issue search queries when needed, read retrieved evidence, and continue reasoning before producing the final answer. We use the same search interface, top-3 retrieved passages per search call, maximum action budget of 4 search rounds, and answer evaluator as the SearchQA environment. Search-o1 does not update model parameters and does not use external skills; it differs from RAG mainly by allowing iterative search and reasoning during inference.

\noindent\textbf{Search-R1 on SearchQA.}
Search-R1 is evaluated as an RL-with-search baseline following the Search-R1 training protocol~\cite{Jin2025SearchR1TL}. During rollout, the policy alternates between model-generated reasoning tokens and search calls, and retrieved tokens are masked from the policy-gradient loss so that optimization is applied only to model-generated tokens. We use the shared SearchQA search interface with top-3 retrieved passages, maximum action budget 4, outcome exact-match reward, learning rate $10^{-6}$, and the same separation between the train split, development split, and test split as \methodname{}. This baseline does not use external skills or lifecycle operations, so it isolates the effect of learning to interact with search from the effect of managing a skill bank.

\noindent\textbf{SFT on SearchQA.}
SFT is evaluated as a supervised fine-tuning baseline on SearchQA~\cite{chung2024scaling}. Training examples are constructed only from the train split using the same response format as the search-enabled environment. The model is optimized to imitate the target reasoning-and-answer trajectory, while validation is used only for checkpoint selection. SFT does not use outcome-reward RL, external skills, lifecycle auditing, or trajectories from the test split during training.

\noindent\textbf{Reject Sampling on SearchQA.}
Reject Sampling follows the Search-R1 baseline construction~\cite{Jin2025SearchR1TL,ahn2024large}. For each training prompt, we sample five search-enabled candidate trajectories with the same search interface and action budget, keep trajectories whose final answer is correct under exact match, and fine-tune the model on the selected trajectories. This preserves the multi-turn LLM--search interaction format while replacing online RL with filtered supervised learning. Validation is used for checkpoint selection, and the test split is held out until reporting.

\section{Evaluation Setup}
\label{app:setup}

\noindent\textbf{Benchmark Protocol.}
We evaluate all methods on the two agent benchmarks used in the main paper. ALFWorld~\cite{shridhar2020alfworld} is a long-horizon text-interaction benchmark in which the agent must complete household tasks by issuing admissible text actions. We report both the overall success rate and task-type success rates for Pick, Look, Clean, Heat, Cool, and Pick2, following the task categories used by the environment. Although ALFWorld has a fixed set of task types, RL supervision is collected over multi-step action trajectories, so each episode contributes many action-level decisions under the shared rollout protocol. SearchQA follows the search-augmented question-answering setting of Search-R1~\cite{Jin2025SearchR1TL}. The agent interacts with a search tool and must eventually output a final answer for questions from NQ~\cite{kwiatkowski2019nq}, TriviaQA~\cite{joshi2017triviaqa}, PopQA~\cite{mallen2023popqa}, HotpotQA~\cite{yang2018hotpotqa}, 2Wiki~\cite{ho20202wiki}, MuSiQue~\cite{trivedi2022musique}, and Bamboogle~\cite{press2023bamboogle}. All methods use the same environment-side action parser, search interface, and trajectory termination rule within each benchmark. The cross-backbone experiments additionally evaluate \methodname{}, SkillRL, and Skill0 on WebShop~\cite{yao2022webshop}, where agents search, inspect, and select products according to user-specified requirements.

\noindent\textbf{Data Splits and Data Usage.}
We keep the train split, development split, and test split separate. ALFWorld follows the official data-split protocol~\cite{shridhar2020alfworld} for household text-interaction tasks. For SearchQA, the train split draws questions from HotpotQA and NQ. The development split draws from HotpotQA, TriviaQA, 2Wiki, NQ, and MuSiQue. The test split draws from PopQA, 2Wiki, TriviaQA, HotpotQA, NQ, MuSiQue, and Bamboogle. WebShop follows the common protocol used in prior work for assigning product-shopping goals to the train split, development split, and test split.
The protocol is fixed and shared across all methods: policy optimization uses the train split, hyperparameters are selected on a smaller subset of the train split, and training-time skill management and validation use the development split. The test split is reserved for final reporting after the policy checkpoint and active skill set are frozen and does not participate in training.

\noindent\textbf{Success Metrics.}
The primary metric in all tables is success rate. Avg. denotes the overall success rate. In ALFWorld, a rollout succeeds when the environment signals completion of the target objective. In SearchQA, success requires a final answer matching the benchmark ground truth under the shared rule. WebShop reports binary success over goals in its test split.

\noindent\textbf{Robustness Diagnostics.}
The main results and training curves report the original runs. Appendix~\ref{app:robustness} additionally reports three independent full training runs of \methodname{}, SkillRL, and Skill0 for each Qwen3-4B/8B and ALFWorld/SearchQA combination. We summarize success rate with the mean and standard deviation across runs and report the initial, maximum, minimum, and final active skill counts for each ALFWorld \methodname{} run. Lifecycle thresholds are selected on a smaller subset of the train split and held fixed across seeds. Bootstrap resampling of test outcomes from the original runs provides a complementary assessment of performance stability.

\noindent\textbf{Checkpoint Selection.}
\methodname{}, SkillRL, Skill0, and the other online-RL baselines use the final checkpoint after the fixed training horizon for each benchmark: 120 GRPO steps on ALFWorld, 180 on SearchQA, and 150 on WebShop. Training proceeds for the full horizon, with validation used for lifecycle updates and monitoring. SFT and Reject Sampling use validation-based checkpoint selection according to their respective supervised training protocols.

\noindent\textbf{Final Evaluation Protocol.}
Final evaluation uses the checkpoint specified above with all training-time adaptation disabled. For \methodname{}, the final active skill set is fixed before test evaluation; test rollouts may retrieve from this frozen active set, but retain, retire, expand, lifecycle audit, and skill creation are disabled. With-skill and no-skill evaluations use the same trained \methodname{} policy; the latter disables skill insertion at inference. For Skill0, the curriculum state is fixed by its schedule before final evaluation. For SkillRL, the skill bank and retrieval configuration are fixed before test evaluation. Prompt-based and agent- or memory-based baselines are also evaluated on the same test split; when a method accumulates memory during evaluation, that memory is created only from earlier evaluation episodes of that method and is not shared across methods.

\noindent\textbf{Fairness Controls.}
All main comparisons use Qwen3-4B~\cite{qwen3technical} as the backbone model. Online-RL methods start from the same base checkpoint and share the benchmark environment, reward, action format, training horizon, and evaluation protocol. Prompt-based, agent, and memory baselines use the same served backbone and environment wrappers, including the action parser and success criterion. We align prompt and response limits, rollout horizons, validation cadence, and test splits where applicable, while retaining each method's validated configuration. In particular, SkillRL uses top-6 retrieval and reference-policy KL regularization, whereas \methodname{} uses top-3 retrieval without actor-side KL. Within-\methodname{} ablations hold the other training settings fixed to assess the lifecycle components.
For methods with skill creation or bank updates, comparable modules use the same underlying components where applicable. With-skill methods share the initial SkillRL-style bank family, the insertion format in Figure~\ref{fig:skill_insertion_prompt}, and the embedding backbone, though retrieval caps remain method-specific. \methodname{} and SkillRL use the same OpenAI-compatible \texttt{o3} skill creator under their respective output formats; EvolveR and memory baselines use the same service-model family where applicable. Methods without skill creation receive no generated skills, and no method creates or updates skills from the test split. Both RL and service-based evaluation use SGLang with shared default inference settings and sampled decoding unless a method's protocol specifies otherwise.

\noindent\textbf{No Cold-Start SFT.}
All RL baselines are evaluated under the same controlled no-warmup setting. Specifically, all RL-based methods start from the same base checkpoint without cold-start SFT or warmup, so the comparison isolates online RL and lifecycle behavior rather than supervised pre-adaptation. This avoids confounding lifecycle effects with a policy that has already been trained to follow a particular skill format, trajectory style, or action pattern before RL begins. We select method-specific hyperparameters on a smaller subset of the train split and fix them before the main training runs.

\noindent\textbf{Leakage Prevention.}
The key separation in our protocol is that training-time control signals are never computed from the test split. In particular, \methodname{} estimates marginal external contribution, routed failure buckets, retain/retire decisions, and expansion triggers only on the development split, never on tasks from the test split. Expansion prompts are constructed strictly from validation-side audit failures; they never include prompts, trajectories, failures, labels, or answers from the test split. Expanded skills are created before final evaluation, and the skill creator never receives prompts, trajectories, failures, or labels from the test split. During final evaluation, the policy checkpoint and active skill set are frozen; test rollouts may retrieve from the frozen active set, but they cannot create, edit, retire, or expand skills.
\methodname{}, Skill0, and SkillRL follow the same separation between the train split, development split, and test split. The development split supports training-time skill management: \methodname{} audits and updates its active skills, Skill0 updates its curriculum, and SkillRL evolves its skill bank. Examples in the test split are held out until reporting. For online memory baselines, any memory accumulated during final evaluation is private to that baseline and arises only from earlier episodes in the same sequential evaluation run; it is not used to select checkpoints, update model parameters, tune prompts, or modify skill banks before evaluation. This prevents labels, trajectories, and failures from the test split from entering policy optimization, skill-bank updates, lifecycle decisions, or skill expansion.

\section{Additional Experimental Results}
\label{app:additional_results}
This appendix further examines the generalization, robustness, and practical behavior of \methodname{}. We study generalization to an additional backbone and benchmark, transfer of the final active skill bank, sensitivity to skill initialization, skill creators, and hyperparameters, and stability across independent training runs. We also provide expanded baseline comparisons and an analysis of audit overhead.

\subsection{Generalization to an Additional Backbone and Benchmark}
\label{app:cross_backbone_generalization}
\begin{table}[t]
\centering
\definecolor{slimBestRed}{RGB}{232,204,198}
\definecolor{slimSecondBlue}{RGB}{202,216,224}
\providecommand{\bestcell}[1]{\cellcolor{slimBestRed}\textbf{#1}}
\providecommand{\secondcell}[1]{\cellcolor{slimSecondBlue}\underline{#1}}
\caption{Generalization to an additional backbone and benchmark. Each method is independently trained for each backbone and benchmark. The left group reports success rate (\%); the right group reports the initial, maximum, minimum, and final active skill counts of the corresponding \methodname{} run. $^\dagger$ denotes inference with external skills. \textbf{\colorbox{slimBestRed}{Best}} success rates are highlighted.}
\label{tab:cross_backbone_generalization}
\vspace{3pt}
\small
\setlength{\tabcolsep}{4pt}
\renewcommand{\arraystretch}{1.12}
\begin{adjustbox}{max width=\textwidth}
\begin{tabular}{llccc@{\hspace{12pt}}rrrr}
\toprule
\multirow{2}{*}{Benchmark} & \multirow{2}{*}{Backbone}
& \multicolumn{3}{c}{Success rate (\%)} & \multicolumn{4}{c}{\methodname{} active skills} \\
\cmidrule(lr){3-5}\cmidrule(lr){6-9}
& & SkillRL$^\dagger$ & Skill0 & \methodname{}$^\dagger$ & Init. & Max. & Min. & Final \\
\midrule
\multirow{2}{*}{ALFWorld}
& Qwen3-4B & 75.0 & 74.2 & \bestcell{87.5} & 38 & 46 & 21 & 21 \\
& Qwen3-8B & 84.7 & 82.1 & \bestcell{89.1} & 38 & 40 & 13 & 15 \\
\midrule
\multirow{2}{*}{SearchQA}
& Qwen3-4B & 38.1 & 39.3 & \bestcell{41.0} & 34 & 37 & 4 & 4 \\
& Qwen3-8B & 41.7 & 41.4 & \bestcell{42.7} & 34 & 36 & 2 & 2 \\
\midrule
\multirow{2}{*}{WebShop}
& Qwen3-4B & 65.8 & 61.2 & \bestcell{71.0} & 57 & 63 & 24 & 28 \\
& Qwen3-8B & 75.6 & 72.4 & \bestcell{79.8} & 57 & 60 & 16 & 19 \\
\bottomrule
\end{tabular}
\end{adjustbox}
\end{table}

We extend the main evaluation with Qwen3-8B as an additional policy backbone and WebShop~\cite{yao2022webshop} as an additional benchmark. We independently train \methodname{}, SkillRL, and Skill0 with both Qwen3-4B and Qwen3-8B on ALFWorld, SearchQA, and WebShop. Table~\ref{tab:cross_backbone_generalization} brings together task performance and the active skill counts of each \methodname{} run, showing how lifecycle management adapts to the policy backbone and task demands. The training and evaluation settings are provided in Appendix~\ref{sec:cross_backbone_setup}.

\noindent\textbf{Performance across Backbones.}
\methodname{} achieves the highest success rate. With the additional Qwen3-8B backbone, \methodname{} reaches 89.1 on ALFWorld and 42.7 on SearchQA, improving over the strongest baseline by 4.4 and 1.0 points, respectively. These results extend the performance advantage of lifecycle management to the larger policy backbone.

\noindent\textbf{Generalization to WebShop.}
The additional WebShop benchmark tests lifecycle management in web interaction, where agents must search for products, inspect their attributes, and select an item that satisfies the task requirements. \methodname{} reaches 71.0 with Qwen3-4B and 79.8 with Qwen3-8B, exceeding the strongest baseline by 5.2 and 4.2 points, respectively. This extends the gains to a third environment with a different interaction structure.

\noindent\textbf{Backbone-Dependent Lifecycle Endpoints.}
The active skill counts in Table~\ref{tab:cross_backbone_generalization} show that every setting first expands beyond its initialization and subsequently contracts. Qwen3-8B retains fewer final skills than Qwen3-4B on all three environments. The larger backbone achieves stronger performance with less persistent external support, consistent with the relationship between parametric knowledge and external support discussed in Appendix~\ref{app:representation_support}. The endpoint also varies by environment. SearchQA converges to an almost empty active set, while ALFWorld and WebShop retain compact sets of procedural skills. Thus, \methodname{} learns an external skill boundary that adapts to both the policy backbone and the task.

\subsection{Transfer of the Final Active Skill Bank}
\label{app:skill_bank_transfer}
\begin{table}[t]
\centering
\definecolor{slimBestRed}{RGB}{232,204,198}
\caption{Transfer evaluation of the final active skill bank learned by \methodname{}. ``None'' uses no external skill, ``Initial skill bank'' uses the pre-training skill bank, and ``Final \methodname{} active skills'' uses the final active set learned by \methodname{}. The transferred skill bank is fixed and no lifecycle update is performed during evaluation.}
\label{tab:skill_bank_transfer}
\providecommand{\bestcell}[1]{\cellcolor{slimBestRed}\textbf{#1}}
\small
\begin{adjustbox}{max width=\textwidth}
\begin{tabular}{llcc}
\toprule
Prompting Method & Skill Source & ALFWorld Avg. & SearchQA Avg. \\
\midrule
\multirow{3}{*}{Zero-Shot} 
& None & 41.4 & 32.3 \\
& Initial skill bank & 63.3 & 31.5 \\
& Final \methodname{} active skills & 65.8 & 32.7 \\
\midrule
\multirow{3}{*}{Few-Shot} 
 & None & 39.1 & 35.5 \\
& Initial skill bank & 64.1 & 34.8 \\
& Final \methodname{} active skills & 66.9 & 35.8 \\
\bottomrule
\end{tabular}
\end{adjustbox}
\end{table}

Table~\ref{tab:skill_bank_transfer} transfers different skill sources to policies that did not participate in lifecycle training: no external skills, the initial skill bank before RL, and the final active skill set learned by \methodname{}. On ALFWorld, the transfer is strong: final \methodname{} skills improve zero-shot and few-shot policies by 24.4 and 27.8 points over no-skill prompting, and by 2.5 and 2.8 points over the initial skill bank. This suggests that the learned active set captures reusable procedural guidance rather than only serving as a private scaffold for the trained \methodname{} policy.
SearchQA shows a different pattern. The final active skills improve zero-shot and few-shot prompting only slightly over no-skill prompting, by 0.4 and 0.3 points, but they consistently outperform the initial skill bank by 1.2 and 1.0 points. Thus, on SearchQA the final skill bank mainly filters noisy or less useful external guidance rather than providing large direct transfer gains. This supports the broader view that skill-bank transferability depends on whether the benchmark benefits from persistent external procedural support.

\subsection{Sensitivity to Skill Initialization}
\label{app:skill_initialization}
\begin{table}[t]
\centering
\caption{Sensitivity to the initial skill bank. The table reports final performance and final lifecycle statistics under different initialization conditions. The noisy bank randomly replaces 30\% of the original skills one-for-one, keeping 38 initial skills.}
\label{tab:skill_initialization}
\small
\begin{adjustbox}{max width=\textwidth}
\begin{tabular}{lcccc}
\toprule
Initial Skill Bank & Final Avg. & Final Active Skills & Retired Skills & Expanded Skills \\
\midrule
Empty skill bank & 76.4 & 18 & 8 & 26 \\
Weak initial skill bank & 81.2 & 23 & 15 & 29 \\
Noisy initial skill bank & 85.6 & 25 & 46 & 33 \\
Original initial skill bank & 87.5 & 21 & 33 & 16 \\
\bottomrule
\end{tabular}
\end{adjustbox}
\end{table}

Table~\ref{tab:skill_initialization} compares four initial banks under the same training and audit protocol. The \emph{empty} setting removes all external skills; when no active task-specific skill is retrieved, validation failures enter an uncovered bucket indexed by task type. This bucket is not part of $\mathcal{A}_t$, but it can trigger the same expansion operator after accumulating $n_{\mathrm{expand}}$ failures with no-skill performance below $\tau_{\mathrm{expand}}$, thereby bootstrapping the first skills when $\mathcal{A}_0=\varnothing$. Later expansions use ordinary skill-routed failure buckets. The \emph{weak} bank retains a stratified 25\% subset of the original skills. The \emph{noisy} bank randomly selects 30\% of the original skills and replaces them one-for-one with mismatched skills, preserving the original bank size of 38. The \emph{original} bank uses the default initialization from the main experiment.

\emph{Expanded Skills} counts distinct skill artifacts accepted into the bank over the full training run, including those later retired. Accordingly, the cumulative number of retired skills can exceed the initial bank size because it may include expanded skills that are subsequently retired. The budget $B=2$ applies to each qualifying audited failure bucket rather than globally to an audit cycle. ALFWorld considers at most $M=4$ task-specific skill-routed or uncovered buckets per audit, so it can accept at most $MB=8$ new skills in one audit.

\noindent\textbf{Robustness to Initialization.}
This experiment varies the initial skill bank while keeping the training and audit protocol fixed. Table~\ref{tab:skill_initialization} shows that \methodname{} is robust to imperfect initial banks because expansion can repair missing coverage and retirement can filter noisy or low-value skills. At the same time, the original initial bank still gives the highest final score, indicating that lifecycle management complements rather than replaces skill initialization.

\noindent\textbf{Expansion from Weak Coverage.}
The empty and weak settings reach 76.4\% and 81.2\% while expanding 26 and 29 skills, respectively. Thus, \methodname{} can construct external support from persistent failures and recover part of the missing task coverage. The remaining gaps to the original setting show that expansion does not fully replace a reasonably informative initial library.

\noindent\textbf{Filtering Noisy Skills.}
The noisy setting reaches 85.6\%, only 1.9 points below the original setting. Starting from 38 skills, it expands 33 and cumulatively retires 46, ending with 25 active skills. This indicates that \methodname{} removes harmful external knowledge and repairs missing coverage rather than blindly preserving the initial bank.

\noindent\textbf{Initialization Still Matters.}
The original skill bank gives the best result, but \methodname{} still substantially reshapes it by retiring 33 skills and expanding 16 new ones, leaving a compact active set of 21 skills. Thus, the gain does not come from static reuse of the initial skills. A reasonable initial bank improves the ceiling, while lifecycle management determines which skills should remain active after RL.

\subsection{Robustness of SLIM Performance}
\label{app:robustness}
\begin{table}[t]
\centering
\definecolor{slimBestRed}{RGB}{232,204,198}
\providecommand{\bestcell}[1]{\cellcolor{slimBestRed}\textbf{#1}}
\caption{Performance over three independent full training runs. Entries are mean $\pm$ standard deviation of success rate (\%). $^\dagger$ denotes inference with external skills. \textbf{\colorbox{slimBestRed}{Best}} results for each backbone and benchmark are highlighted.}
\label{tab:training_run_robustness}
\small
\begin{adjustbox}{max width=\textwidth}
\begin{tabular}{llcc}
\toprule
Backbone & Method & ALFWorld & SearchQA \\
\midrule
\multirow{3}{*}{Qwen3-4B}
& Skill0 & $73.44\pm0.78$ & $39.39\pm0.19$ \\
& SkillRL$^\dagger$ & $76.30\pm1.63$ & $38.15\pm0.24$ \\
& \methodname{}$^\dagger$ & \bestcell{$\bm{88.02\pm1.63}$} & \bestcell{$\bm{41.37\pm0.39}$} \\
\midrule
\multirow{3}{*}{Qwen3-8B}
& Skill0 & $82.55\pm0.45$ & $41.45\pm0.14$ \\
& SkillRL$^\dagger$ & $85.68\pm1.19$ & $41.85\pm0.17$ \\
& \methodname{}$^\dagger$ & \bestcell{$\bm{89.32\pm1.19}$} & \bestcell{$\bm{42.70\pm0.22}$} \\
\bottomrule
\end{tabular}
\end{adjustbox}
\end{table}

\begin{table}[t]
\centering
\caption{ALFWorld skill lifecycle statistics of \methodname{} over three independent training runs for each backbone. Init., Max., Min., and Final denote the initial, maximum, minimum, and final active skill counts.}
\label{tab:lifecycle_run_robustness}
\small
\begin{adjustbox}{max width=\textwidth}
\begin{tabular}{lrrrrr}
\toprule
Backbone & Run & Init. & Max. & Min. & Final \\
\midrule
\multirow{3}{*}{Qwen3-4B}
& 1 & 38 & 46 & 21 & 21 \\
& 2 & 38 & 53 & 17 & 20 \\
& 3 & 38 & 43 & 19 & 19 \\
\midrule
\multirow{3}{*}{Qwen3-8B}
& 1 & 38 & 40 & 13 & 15 \\
& 2 & 38 & 48 & 12 & 14 \\
& 3 & 38 & 45 & 11 & 15 \\
\bottomrule
\end{tabular}
\end{adjustbox}
\end{table}

\begin{table}[t]
\centering
\caption{Bootstrap robustness of \methodname{} gains over the strongest skill-based baselines. Gaps are success-rate percentage points. Confidence intervals are computed with 10,000 independent aggregate bootstrap resamples over test outcomes from original runs.}
\label{tab:slim_performance_robustness}
\small
\begin{adjustbox}{max width=\textwidth}
\begin{tabular}{llccc}
\toprule
Benchmark & Comparison & Mean Gap & 95\% CI & Crosses 0? \\
\midrule
ALFWorld & \methodname{}$^\dagger$ -- Skill0 & +13.3 & [+3.9, +22.7] & No \\
ALFWorld & \methodname{}$^\dagger$ -- SkillRL$^\dagger$ & +12.6 & [+3.1, +21.9] & No \\
SearchQA & \methodname{}$^\dagger$ -- Skill0 & +1.72 & [+1.13, +2.31] & No \\
SearchQA & \methodname{}$^\dagger$ -- SkillRL$^\dagger$ & +2.95 & [+2.35, +3.52] & No \\
\bottomrule
\end{tabular}
\end{adjustbox}
\end{table}

We examine robustness through independent training runs, the resulting skill lifecycle trajectories, and bootstrap resampling of test outcomes. Together, these experiments characterize the stability of both task performance and the learned external skill boundary.

\noindent\textbf{Performance across Training Runs.}
We conduct three independent full training runs for \methodname{}, SkillRL, and Skill0 with Qwen3-4B and Qwen3-8B on ALFWorld and SearchQA. Table~\ref{tab:training_run_robustness} reports the mean and standard deviation across these runs; Table~\ref{tab:main_results} reports the original run. \methodname{} achieves the highest mean success rate in every setting. On ALFWorld, it reaches $88.02\pm1.63$ with Qwen3-4B and $89.32\pm1.19$ with Qwen3-8B. On SearchQA, it reaches $41.37\pm0.39$ and $42.70\pm0.22$, exceeding the strongest baseline by 1.98 and 0.85 points, respectively. The gains persist across independent runs and both backbones, showing that lifecycle-guided training consistently improves performance in the two environments.

\noindent\textbf{Stability of Lifecycle Endpoints.}
Table~\ref{tab:lifecycle_run_robustness} summarizes the ALFWorld active-set trajectory of each \methodname{} run. Stochastic task order and rollouts produce different intermediate maxima and minima, while all runs follow the same qualitative pattern of expansion and retirement. Qwen3-4B finishes with 19--21 active skills, and Qwen3-8B finishes with 14--15. The compact endpoints remain stable across runs even as the intermediate trajectories vary. Together with Table~\ref{tab:cross_backbone_generalization}, this shows that the learned boundary is reproducible within a backbone and adapts across backbones.

\noindent\textbf{Robustness over Test Outcomes.}
Table~\ref{tab:slim_performance_robustness} complements the training-run analysis with an independent aggregate bootstrap over test outcomes from the original runs. All confidence intervals for the performance gaps remain above zero. On ALFWorld, the intervals are wider, but the lower bounds remain positive against both Skill0 and SkillRL$^\dagger$. On SearchQA, the intervals are tighter. The bootstrap results and repeated training runs consistently support the performance advantage of \methodname{}.

\subsection{Sensitivity to the Skill Creator}
\label{app:skill_creator_sensitivity}
\begin{table}[t]
\centering
\caption{Sensitivity to the skill creator with Qwen3-4B as the policy backbone. Each row reports the success rate (\%) of a full \methodname{} training run with the specified creator and inference with the final active skills.}
\label{tab:skill_creator_sensitivity}
\small
\begin{adjustbox}{max width=\textwidth}
\begin{tabular}{lcc}
\toprule
Skill Creator & ALFWorld & SearchQA \\
\midrule
Frozen Qwen3-4B & 86.7 & 41.1 \\
DeepSeek-V4-Flash & 88.3 & 40.1 \\
GPT-4o & 85.9 & 40.9 \\
o3 & 87.5 & 41.0 \\
\bottomrule
\end{tabular}
\end{adjustbox}
\end{table}

We use o3 in the main experiments to match SkillRL's default skill-creator configuration. To examine how the choice of creator affects \methodname{}, we fix the policy backbone to Qwen3-4B and replace o3 with frozen Qwen3-4B, DeepSeek-V4-Flash, and GPT-4o while keeping the skill-creation prompt, lifecycle rules, and training settings fixed. Table~\ref{tab:skill_creator_sensitivity} reports the results on ALFWorld and SearchQA.

\noindent\textbf{Performance across Skill Creators.}
Changing the skill creator has only a modest effect on \methodname{}. Success rates range from 85.9 to 88.3 on ALFWorld and from 40.1 to 41.1 on SearchQA across the four creators. Replacing o3 with a frozen Qwen3-4B creator changes the scores by only $-0.8$ and $+0.1$ points, respectively, reaching 86.7 and 41.1. On ALFWorld, this setting exceeds SkillRL$^\dagger$ with o3 by 11.7 points. These results show that \methodname{} preserves its performance advantage across different skill creators, including one with the same backbone as the policy.

\noindent\textbf{Contribution of Skill Expansion.}
The no-expansion variant in Table~\ref{tab:ablation_lifecycle} retains MEC-driven retention and retirement and reaches 78.9 on ALFWorld without creator calls, exceeding SkillRL$^\dagger$ by 3.9 points and Skill0 by 4.7 points. Enabling expansion adds 7.8 points with a frozen Qwen3-4B creator and 8.6 points with o3. Together, the ablation and creator substitutions attribute \methodname{}'s gains to lifecycle management and subsequent policy adaptation, rather than distillation from o3.

\subsection{Sensitivity to Hyperparameters}
\label{app:hyperparameter_sensitivity}
\begin{figure}[t]
\centering
\includegraphics[width=\linewidth]{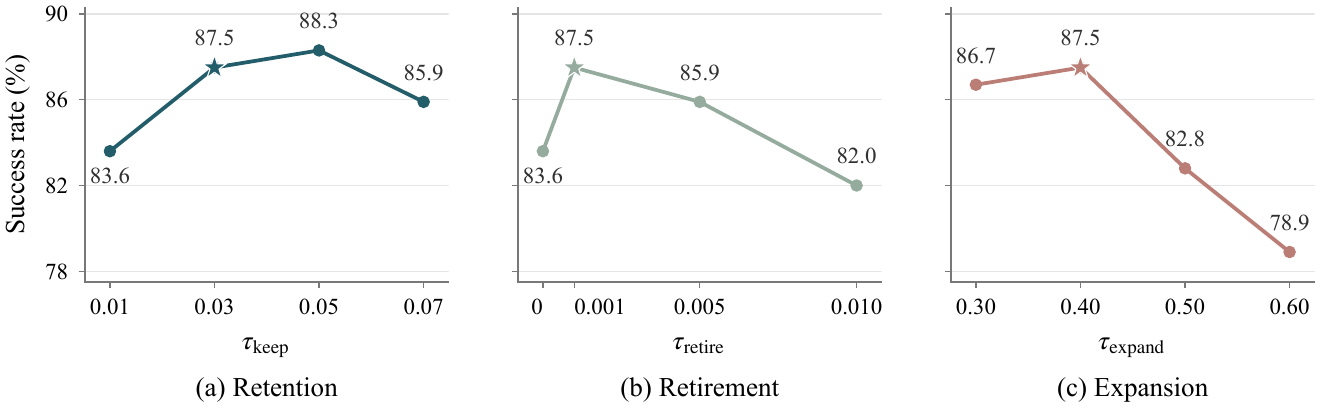}
\caption{Sensitivity to hyperparameters on ALFWorld. Each panel varies one lifecycle threshold while fixing the other hyperparameters. Stars mark the default values: $\tau_{\mathrm{keep}}=0.03$, $\tau_{\mathrm{retire}}=0.001$, and $\tau_{\mathrm{expand}}=0.40$. All panels use the same success-rate scale.}
\label{fig:hyperparameter_sensitivity}
\end{figure}

We study the three hyperparameters that govern retain, retire, and expand decisions on ALFWorld. We vary one threshold at a time and fix the others at their defaults: $\tau_{\mathrm{keep}}=0.03$, $\tau_{\mathrm{retire}}=0.001$, and $\tau_{\mathrm{expand}}=0.40$. Other settings remain fixed. Figure~\ref{fig:hyperparameter_sensitivity} shows the resulting success rates.

\noindent\textbf{Retention Threshold.}
Setting $\tau_{\mathrm{keep}}$ between 0.03 and 0.07 yields 85.9--88.3, with the highest score at 0.05. Lowering it to 0.01 gives 83.6. A smaller threshold allows weakly positive skills to satisfy the retain condition, making retention more permissive. The stronger results at intermediate values support selective retention of skills with clear external value.

\noindent\textbf{Retirement Threshold.}
The intermediate values 0.001 and 0.005 yield 87.5 and 85.9. Setting $\tau_{\mathrm{retire}}=0$ makes retirement more conservative and lowers success to 83.6, while increasing it to 0.010 makes retirement more aggressive and lowers success to 82.0. The two-sided decrease reflects the balance between removing redundant support and preserving skills with positive external value.

\noindent\textbf{Expansion Threshold.}
Values of 0.30 and 0.40 yield 86.7 and 87.5. Increasing $\tau_{\mathrm{expand}}$ to 0.50 and 0.60 makes more routed regions eligible for expansion and lowers success to 82.8 and 78.9. Focused expansion therefore provides more effective capability coverage. Across the three thresholds, \methodname{} maintains strong performance over useful operating ranges: $\tau_{\mathrm{keep}}\in[0.03,0.07]$, $\tau_{\mathrm{retire}}\in[0.001,0.005]$, and $\tau_{\mathrm{expand}}\in[0.30,0.40]$.

\subsection{Expanded Results}
\label{app:expanded_results}
\begin{table}[t]
\centering
\definecolor{slimBestRed}{RGB}{232,204,198}
\caption{Expanded ALFWorld comparison. All entries report task success rate. $^\dagger$ denotes evaluation with retrieved external skills. $^*$ denotes closed-source model results copied from SkillRL~\cite{xia2026skillrl}. Avg. denotes the overall success rate. \textbf{\colorbox{slimBestRed}{Best}} results are highlighted.}
\label{tab:expanded_alfworld}
\providecommand{\bestcell}[1]{\cellcolor{slimBestRed}\textbf{#1}}
\small
\begin{adjustbox}{max width=\textwidth}
\begin{tabular}{l|cccccc|c}
\toprule
Method & Pick & Look & Clean & Heat & Cool & Pick2 & Avg. \\
\midrule
GPT-4o$^*$ & 75.3 & 60.8 & 31.2 & 56.7 & 21.6 & 49.8 & 48.2 \\
Gemini-2.5-Pro$^*$ & 92.8 & 63.3 & 62.1 & 69.0 & 26.6 & 58.7 & 64.3 \\
SimpleMem & \bestcell{100.0} & 25.0 & 9.5 & 53.3 & 39.1 & 11.8 & 45.0 \\
MemRL & \bestcell{100.0} & 33.3 & 10.0 & 46.2 & 9.5 & 6.7 & 38.6 \\
RLOO & 85.3 & \bestcell{100.0} & 50.0 & 75.0 & 52.2 & 38.9 & 62.8 \\
\methodname{}$^\dagger$ & 92.9 & \bestcell{100.0} & \bestcell{91.4} & \bestcell{78.3} & \bestcell{88.5} & \bestcell{81.2} & \bestcell{87.5} \\
\bottomrule
\end{tabular}
\end{adjustbox}
\end{table}

\begin{table}[t]
\centering
\definecolor{slimBestRed}{RGB}{232,204,198}
\caption{Expanded SearchQA comparison. All entries report success rate. $^\dagger$ denotes evaluation with retrieved external skills. Avg. denotes the overall success rate. \textbf{\colorbox{slimBestRed}{Best}} results are highlighted.}
\label{tab:expanded_searchqa}
\providecommand{\bestcell}[1]{\cellcolor{slimBestRed}\textbf{#1}}
\small
\begin{adjustbox}{max width=\textwidth}
\begin{tabular}{lcccccccc}
\toprule
Method & NQ & TriviaQA & PopQA & HotpotQA & 2Wiki & MuSiQue & Bamboogle & Avg. \\
\midrule
RAG & 37.95 & 58.62 & \bestcell{44.99} & \bestcell{37.10} & 25.40 & 7.57 & 21.60 & 39.78 \\
Search-o1 & 30.60 & 51.80 & 22.60 & 28.20 & 38.20 & 11.20 & \bestcell{41.60} & 33.66 \\
Search-R1 & 37.56 & 51.05 & 38.36 & 33.68 & \bestcell{40.09} & \bestcell{13.20} & 35.20 & 39.65 \\
SFT & \bestcell{48.70} & 49.40 & 43.60 & 28.80 & 31.40 & 7.10 & 32.00 & 38.40 \\
Reject Sampling & 37.98 & 48.36 & 38.36 & 30.83 & 30.98 & 9.97 & 37.60 & 36.32 \\
\methodname{}$^\dagger$ & 38.40 & \bestcell{62.10} & 40.40 & 36.90 & 31.50 & 12.70 & 36.00 & \bestcell{41.04} \\
\bottomrule
\end{tabular}
\end{adjustbox}
\end{table}

\noindent\textbf{Additional ALFWorld Baselines.}
The expanded ALFWorld comparison adds closed-source, memory-based, and RL-based baselines beyond the main table. \methodname{} remains competitive against these broader alternatives while using the same environment and evaluation protocol, which strengthens the conclusion that lifecycle management improves procedural agent training rather than only outperforming a narrow baseline set.

\noindent\textbf{Additional SearchQA Baselines.}
The expanded SearchQA comparison includes retrieval-augmented, supervised, rejection-sampling, and RL-style baselines. These results show that \methodname{} is not merely benefiting from a stronger prompting or search interface; its gains remain consistent when compared with alternative ways of improving search-augmented QA behavior.

\subsection{Audit Overhead Comparison}
\label{app:audit_overhead}
\begin{table*}[t]
\centering
\caption{Audit overhead comparison on the ALFWorld training setting. $V$ denotes one ordinary validation pass, $S$ is the full skill-bank size, $S_k$ is the task-specific active pool for task type $k$, and $K$ is the bounded SLIM audit budget.}
\label{tab:audit_overhead}
\small
\begin{adjustbox}{max width=\textwidth}
\begin{tabular}{lp{0.18\textwidth}p{0.18\textwidth}p{0.22\textwidth}p{0.22\textwidth}}
\toprule
Method & Additional validation-equivalent calls & Skill-bank scaling & Bottleneck & Main benefit \\
\midrule
Skill0 & $O(1)$ coarse curriculum comparisons & Shrinking skill set & Curriculum validation & Efficient withdrawal toward zero-skill inference \\
\midrule
SkillRL & $O(1)$ ordinary validation & $O(S)$ retrieval / maintenance as skills accumulate & Growing skill bank and prompt management & Persistent skill augmentation and exploration support \\
\midrule
\methodname{} & $O(1+K), K\le5$ & $O(S_k)$ retrieval and bounded audit & Periodic leave-one-skill-out audit & Dynamic external boundary through retain, retire, and expand \\
\bottomrule
\end{tabular}
\end{adjustbox}
\end{table*}

\methodname{} introduces leave-one-skill-out audit calls, but the audit budget is periodic and capped. Table~\ref{tab:audit_overhead} shows that \methodname{} is more expensive than ordinary validation, but its lifecycle cost is bounded by $K$ rather than the full skill-bank size. In contrast, SkillRL has lighter validation but can accumulate a growing skill bank, while Skill0 is cheaper but mainly models progressive withdrawal.
We quantify the operational footprint through active skill counts, audit calls, and wall-clock time. \methodname{} reduces the final active skill count from 38 initial skills plus 16 expansions to 21 active skills, while SkillRL grows to 73. Its audit budget is capped and does not scale with the full skill bank, and wall-clock remains comparable to SkillRL and Skill0. Together with the w/o Retirement and Fixed Active Set Size ablations, these results indicate that \methodname{}'s gains are not due to unbounded external support or simple prompt-budget control.

\section{Prompts}
\label{app:prompts}
This section provides the prompt templates used for environment interaction, skill insertion, baseline agents, and skill expansion. We keep placeholders such as \texttt{\{task\_description\}} and \texttt{\{retrieved\_memories\}} explicit so that the templates can be mapped directly to the training and evaluation protocol.

\noindent\textbf{Environment Interaction Prompts.}
Figure~\ref{fig:alfworld_skill_prompt} and Figure~\ref{fig:searchqa_skill_prompt} show the skill-conditioned rollout prompts used for ALFWorld and SearchQA. Figure~\ref{fig:alfworld_base_prompt} and Figure~\ref{fig:searchqa_base_prompt} show the corresponding no-skill templates.

\noindent\textbf{Skill Insertion Format.}
Figure~\ref{fig:skill_insertion_prompt} shows how retrieved general and task-specific skills are inserted into the agent context.

\noindent\textbf{Baseline Prompts.}
Figure~\ref{fig:fewshot_prompt}, Figure~\ref{fig:react_prompt}, Figure~\ref{fig:reflexion_prompt}, Figure~\ref{fig:expel_prompt}, and Figure~\ref{fig:mem0_prompt} provide the prompts used by prompt-based and memory-style baselines.

\noindent\textbf{Skill Creation Prompt.}
Figure~\ref{fig:skill_creation_prompt} gives the expansion prompt used by \methodname{} to create new standalone task-specific \texttt{SKILL.md} artifacts from routed failures.

\section{Skill Bank Details}
\label{app:skill_bank}
Tables~\ref{tab:alfworld_skill_bank} and~\ref{tab:searchqa_skill_bank} summarize representative skills from the hierarchical skill banks used by \methodname{} on ALFWorld and SearchQA, respectively. They include general and task-specific skills with concise trigger conditions and one-line procedural content. Dynamically expanded skills are inserted into the task-specific pool of the corresponding task type and become eligible for later retrieval and lifecycle auditing.

\section{Limitations}
\label{app:limitations}
\methodname{} has three main limitations. First, marginal external contribution is a local single-skill leave-one-out estimate conditioned on the current policy, routing behavior, and active set. It is not a global Shapley-style attribution and does not capture high-order interactions among skills. Second, lifecycle thresholds and audit budgets require tuning on the train split, so transferring the same configuration to substantially different domains may require additional calibration. Third, lifecycle auditing remains practical in our current setting but may become expensive for very large skill banks, where more scalable audit candidate selection would be needed.

\clearpage

\begin{figure*}[h]
\begin{AIbox}{}
{\bf Title: ALFWorld No-Skill Rollout Prompt} \\
{
\small
\texttt{You are an expert agent operating in the ALFRED Embodied Environment. Your task is to: \{task\_description\}}\\
\texttt{Prior to this step, you have already taken \{step\_count\} step(s). Below are the most recent \{history\_length\} observations and the corresponding actions you took: \{action\_history\}}\\
\texttt{You are now at step \{current\_step\} and your current observation is: \{current\_observation\}}\\
\texttt{Your admissible actions of the current situation are: [\{admissible\_actions\}].}\\
\texttt{Now it's your turn to take an action. You should first reason step-by-step within <think> </think> tags, then choose one admissible action within <action> </action> tags.}
}
\end{AIbox}
\caption{ALFWorld no-skill rollout prompt.}
\label{fig:alfworld_base_prompt}
\end{figure*}

\begin{figure*}[h]
\begin{AIbox}{}
{\bf Title: ALFWorld Skill-Conditioned Rollout Prompt} \\
{
\small
\texttt{You are an expert agent operating in the ALFRED Embodied Environment. Your task is to: \{task\_description\}}\\
\texttt{\{retrieved\_memories\}}\\
\texttt{Prior to this step, you have already taken \{step\_count\} step(s). Below are the most recent \{history\_length\} observations and the corresponding actions you took: \{action\_history\}}\\
\texttt{You are now at step \{current\_step\} and your current observation is: \{current\_observation\}}\\
\texttt{Your admissible actions of the current situation are: [\{admissible\_actions\}].}\\
\texttt{Now it's your turn to take an action. You should first reason step-by-step within <think> </think> tags, then choose one admissible action within <action> </action> tags.}
}
\end{AIbox}
\caption{ALFWorld skill-conditioned rollout prompt. The placeholder \texttt{\{retrieved\_memories\}} is filled by the skill insertion format in Figure~\ref{fig:skill_insertion_prompt}.}
\label{fig:alfworld_skill_prompt}
\end{figure*}

\begin{figure*}[h]
\begin{AIbox}{}
{\bf Title: SearchQA No-Skill Rollout Prompt} \\
{
\small
\texttt{You are an expert agent tasked with answering the given question step-by-step.}\\
\texttt{Your question: \{task\_description\}}\\
\texttt{Prior to this step, you have already taken \{step\_count\} step(s). History: \{memory\_context\}}\\
\texttt{Now it's your turn to respond. First reason within <think> </think> tags. Then choose one action: search using <search> your query </search>, or answer using <answer> final answer </answer>.}
}
\end{AIbox}
\caption{SearchQA no-skill rollout prompt.}
\label{fig:searchqa_base_prompt}
\end{figure*}

\begin{figure*}[h]
\begin{AIbox}{}
{\bf Title: SearchQA Skill-Conditioned Rollout Prompt} \\
{
\small
\texttt{You are an expert agent tasked with answering the given question step-by-step.}\\
\texttt{Your question: \{task\_description\}}\\
\texttt{\{retrieved\_memories\}}\\
\texttt{Prior to this step, you have already taken \{step\_count\} step(s). History: \{memory\_context\}}\\
\texttt{Now it's your turn to respond. First reason within <think> </think> tags. Then choose one action: search using <search> your query </search>, or answer using <answer> final answer </answer>.}
}
\end{AIbox}
\caption{SearchQA skill-conditioned rollout prompt.}
\label{fig:searchqa_skill_prompt}
\end{figure*}

\begin{figure*}[h]
\begin{AIbox}{}
{\bf Title: Retrieved Skill Insertion Format} \\
{
\small
\texttt{\#\#\# GENERAL SKILLS \#\#\#}\\
\texttt{Here are some useful general strategies for solving the task effectively:}\\
\texttt{1. \{title\}: \{principle\} Apply this when \{when\_to\_apply\}}\\
\texttt{...}\\[2pt]
\texttt{\#\#\# TASK: \{task\_type\} \#\#\#}\\
\texttt{For \{readable\_task\} tasks, apply these specific strategies:}\\
\texttt{1. \{title\}: \{principle\} Apply this when \{when\_to\_apply\}}\\
\texttt{...}
}
\end{AIbox}
\caption{Skill insertion format used by \methodname{} and skill-conditioned baselines. General skills are inserted as a group when active, while task-specific skills are retrieved by task type and semantic similarity.}
\label{fig:skill_insertion_prompt}
\end{figure*}

\begin{figure*}[h]
\begin{AIbox}{}
{\bf Title: Zero-Shot and Few-Shot Prompting} \\
{
\small
\texttt{Zero-shot:}\\
\texttt{\{obs\_text\}}\\[2pt]
\texttt{Few-shot prefix:}\\
\texttt{Below are solved examples. Follow the same output format and decision style.}\\
\texttt{[Example 1] User: \{example\_user\} Assistant: \{example\_assistant\}}\\
\texttt{[Example 2] User: \{example\_user\} Assistant: \{example\_assistant\}}\\
\texttt{[Current Task]}\\
\texttt{\{obs\_text\}}
}
\end{AIbox}
\caption{Zero-shot and few-shot prompting templates. Skill-conditioned variants insert the retrieved skill block from Figure~\ref{fig:skill_insertion_prompt} into the current task context.}
\label{fig:fewshot_prompt}
\end{figure*}

\begin{figure*}[h]
\begin{AIbox}{}
{\bf Title: ReAct Prompt} \\
{
\small
\texttt{System: You are following the ReAct paradigm: interleave explicit reasoning with environment actions. Reason concisely but concretely. Always output reasoning inside <think>...</think> and the single next action inside the required action tags.}\\[2pt]
\texttt{User: \{obs\_text\}}
}
\end{AIbox}
\caption{ReAct prompt used under the shared environment protocol.}
\label{fig:react_prompt}
\end{figure*}

\begin{figure*}[h]
\begin{AIbox}{}
{\bf Title: Reflexion Prompt} \\
{
\small
\texttt{System: You are using the Reflexion style. Before acting, consult prior reflective notes about common mistakes. Avoid repeating known errors. Keep the output format exact.}\\[2pt]
\texttt{User: Past reflections from related mistakes:}\\
\texttt{- \{reflection\_1\}}\\
\texttt{- \{reflection\_2\}}\\
\texttt{\{obs\_text\}}\\[2pt]
\texttt{Reflection generation: Write one short actionable reflection that would help avoid the same failure next time.}
}
\end{AIbox}
\caption{Reflexion prompt and reflection-generation instruction.}
\label{fig:reflexion_prompt}
\end{figure*}

\begin{figure*}[h]
\begin{AIbox}{}
{\bf Title: ExpeL Prompt} \\
{
\small
\texttt{System: You are using the ExpeL style: consult reusable lessons distilled from previous trajectories, then act using the current environment state.}\\[2pt]
\texttt{User: Retrieved experience lessons:}\\
\texttt{- \{lesson\_1\}}\\
\texttt{- \{lesson\_2\}}\\
\texttt{\{obs\_text\}}\\[2pt]
\texttt{Lesson distillation: Distill one compact reusable lesson from this trajectory. It should be transferable and phrased as an actionable rule.}
}
\end{AIbox}
\caption{ExpeL prompt and lesson-distillation instruction.}
\label{fig:expel_prompt}
\end{figure*}

\begin{figure*}[h]
\begin{AIbox}{}
{\bf Title: Mem0 Prompt} \\
{
\small
\texttt{System: You are using a Mem0-style external memory. Retrieve small factual or procedural memories from prior tasks and use them only when relevant.}\\[2pt]
\texttt{User: Relevant memories from prior tasks:}\\
\texttt{- \{memory\_1\}}\\
\texttt{- \{memory\_2\}}\\
\texttt{\{obs\_text\}}\\[2pt]
\texttt{Memory extraction: Extract up to 3 short atomic memories from this trajectory and return them as a JSON list of strings.}
}
\end{AIbox}
\caption{Mem0 memory prompt and memory-extraction instruction.}
\label{fig:mem0_prompt}
\end{figure*}

\begin{figure*}[h]
\begin{AIbox}{}
{\bf Title: SLIM Skill Creation Prompt} \\
{
\small
\texttt{System: Create one new standalone task-specific SKILL.md artifact that addresses the dominant failure pattern. Do not rewrite an existing skill and do not create a general or foundational skill.}\\[2pt]
\texttt{User: Task type: \{task\_type\}}\\
\texttt{Failed tasks and traces: \{failure\_bucket\}}\\
\texttt{Currently selected skills that were insufficient: \{active\_skills\}}\\
\texttt{Output a skill with:}\\
\texttt{1. Skill name}\\
\texttt{2. Routing-oriented description}\\
\texttt{3. Trigger condition}\\
\texttt{4. Concise procedural workflow}\\
\texttt{Reject generic names, duplicate titles, overly short workflows, and any output that targets general skills.}
}
\end{AIbox}
\caption{Skill creation prompt used by \methodname{} during expansion. The prompt follows an Anthropic-style skill-creator workflow and produces new task-specific \texttt{SKILL.md} artifacts.}
\label{fig:skill_creation_prompt}
\end{figure*}

\clearpage
\begin{table*}[t]
\centering
\caption{Representative ALFWorld skills used by \methodname{}, including the skills analyzed in the lifecycle case study.}
\label{tab:alfworld_skill_bank}
\scriptsize
\begin{adjustbox}{max width=\textwidth}
\begin{tabular}{lllp{0.27\textwidth}p{0.42\textwidth}}
\toprule
Task & Skill ID & Skill Name & Trigger & Content \\
\midrule
General & gen\_004 & Track Counts \& Progress & Multi-instance goals such as putting two objects. & Maintain a counter of remaining goal objects and stop only when the count reaches zero. \\
General & gen\_011 & Efficient Relation Search & Goals mentioning both a target and reference object. & Search one object near the other instead of treating them independently. \\
General & gen\_001 & Systematic Exploration & Goal object count is unmet and unexplored locations remain. & Search plausible surfaces or containers once before revisiting checked locations. \\
General & gen\_002 & Immediate Acquisition & First visual confirmation of a goal-relevant object. & Take a required object as soon as it becomes visible and reachable. \\
General & gen\_003 & Destination First Policy & Holding a goal object after identifying its target location. & Navigate directly to the receptacle and place the object before resuming search. \\
\midrule
pick\_and\_place & pic\_001 & Systematic First-Pass Search & Before acquiring required objects. & Maintain a checklist of visible and closed candidates and inspect each once. \\
pick\_and\_place & pic\_002 & Grab When Seen & First sight of an unheld target object. & Immediately take a needed visible object before moving elsewhere. \\
clean & cle\_003 & Sink First for Cleaning & Target object is held and must be clean. & Go to the nearest sink or basin and clean the object before placement. \\
clean & cle\_005 & State Verification Before Drop & After cleaning and before final placement. & Verify the object is clean; clean again if the state is uncertain. \\
heat & hea\_003 & Open Then Heat & At the microwave with the target object held. & Open the microwave, place the object inside, and execute the heat action. \\
cool & coo\_005 & Direct Post-Cooling Delivery & Cooling action succeeds. & Deliver the cooled object directly to the destination without detours. \\
cool & coo\_002 & Confirm Object Match & After acquiring an object in a cooling task. & Verify that the held item matches the requested target type; otherwise drop it and resume search. \\
cool & coo\_004 & Enforce Cooling Before Placement & Holding the correct object before final placement. & Do not place the target object until a fridge or freezer cooling action has succeeded. \\
cool & \makecell[l]{dyn\_verify\\\_cooling\\\_completion} & Verify Cooling Completion & After placing an object inside a cooling appliance. & Confirm the object is cool before retrieval, then proceed directly to delivery. \\
\bottomrule
\end{tabular}
\end{adjustbox}
\end{table*}

\begin{table*}[t]
\centering
\caption{Representative SearchQA skills used by \methodname{}.}
\label{tab:searchqa_skill_bank}
\scriptsize
\begin{adjustbox}{max width=\textwidth}
\begin{tabular}{lllp{0.27\textwidth}p{0.42\textwidth}}
\toprule
Task & Skill ID & Skill Name & Trigger & Content \\
\midrule
General & gen\_001 & Decompose Then Search & Complex or multi-part questions. & Break the question into minimal sub-questions and search each before synthesis. \\
General & gen\_002 & Precision Query Crafting & Initial query formulation. & Use exact entity names and target attributes while avoiding filler words. \\
General & gen\_003 & Iterative Query Refinement & First result set lacks definitive evidence. & Add qualifiers, alternate names, dates, or context instead of repeating the same query. \\
General & gen\_004 & Source-Backed Assertions & Before committing to a final answer. & Answer only after locating supporting evidence. \\
General & gen\_005 & Cross-Check Multiple Sources & Facts may be outdated or disputed. & Validate names, dates, and numbers against multiple independent sources. \\
\midrule
direct\_retrieval & dir\_001 & Isolate Core Query & Start of direct retrieval. & Strip the question to the key entity and sought fact, then search that pair. \\
direct\_retrieval & dir\_002 & Refine When Empty & Initial search gives weak or no hits. & Reformulate with synonyms, alternate names, dates, or quoted phrases. \\
direct\_retrieval & dir\_003 & Anchor With Quotes & Distinctive phrases, lyrics, titles, or quotes. & Wrap unique phrases in quotation marks to retrieve exact-match sources. \\
direct\_retrieval & \makecell[l]{dyn\_direct\_retrieval\\\_search\_first\\\_answer\_second} & Search First, Answer Second & Factoid or definition queries. & Always run an external search before relying on memory. \\
multi\_hop\_reasoning & mul\_001 & Decompose Question First & Multi-hop questions linking multiple facts. & Split the question into explicit sub-questions before searching. \\
multi\_hop\_reasoning & mul\_003 & Collect-Then-Compare & Comparative multi-hop tasks. & Retrieve concrete values for all items before comparing. \\
entity\_attribute\_lookup & ent\_001 & Direct Attribute Query & Full, unambiguous entity name is given. & Include the full entity name and target attribute in the first search. \\
entity\_attribute\_lookup & ent\_003 & Two-Source Cross-Check & First plausible answer appears or attribute is uncertain. & Confirm the attribute in at least two independent sources. \\
compare & com\_001 & Decompose \& Isolate & Reading a comparison-type question. & Split the question into entities and the single attribute to compare. \\
compare & com\_003 & Normalize Before Comparing & After gathering each entity's attribute. & Convert values to a common comparable form before judging equality or order. \\
\bottomrule
\end{tabular}
\end{adjustbox}
\end{table*}

\end{document}